\documentclass[twoside]{article}

\usepackage[accepted]{aistats2024}

\RequirePackage{amsthm,amsmath,amsfonts,amssymb}
\RequirePackage{titletoc}
\RequirePackage{graphicx}
\usepackage{subfigure}

\usepackage[utf8]{inputenc} 
\usepackage[T1]{fontenc}    
\usepackage[switch, modulo]{lineno}
\usepackage{url}            
\usepackage{booktabs}       
\usepackage{nicefrac}       
\usepackage{microtype}      
\usepackage{algpseudocode}
\usepackage{algorithm}
\usepackage{dsfont}
\usepackage{mathtools}
\usepackage{mathrsfs}
\usepackage[colorlinks=true,citecolor=blue]{hyperref}
\usepackage[round]{natbib}

\usepackage[textsize=tiny,textwidth=20mm]{todonotes}  

\usepackage{colortbl}
\definecolor{lightgray}{rgb}{0.9,0.9,0.9}
\definecolor{lightred}{rgb}{1,0.8,0.8}
\definecolor{lightgreen}{rgb}{0.6,1,0.6}
\definecolor{lightyellow}{rgb}{1,1,0.5}
\definecolor{lightgrey}{rgb}{0.8,0.8,0.8}

\newcommand{\prob}{\mathbb{P}}

\newcommand{\scal}{{\cal S}}

\newcommand{\dcal}{{\mathscr D}}

\newcommand{\ncal}{{\cal N}}

\newcommand{\ecal}{{\mathscr E}}

\newcommand{\beq}{\begin{equation}}
\newcommand{\eeq}{\end{equation}}

\newcommand{\one}{\mathbf{1}}

\def\rset{\mathbb{R}}

\def\zset{\mathbb{Z}}
\def\nset{\mathbb{N}}

\newcommand{\ih}{h}
\newcommand{\ihp}{g}
\newcommand*{\mhegz}[1]{\textbf{\textcolor{magenta}{#1}}}

\newcommand{\QD}{AINQ}

\newcommand{\ceil}[1]{\left \lceil #1 \right \rceil }
\newcommand{\norm}[1]{\left \lVert #1 \right \rVert}
\newcommand{\open}[1]{\left ( #1 \right )}
\newcommand{\closed}[1]{\left [#1 \right]}
\newcommand{\ens}[1]{\left \{ #1 \right \} }
\newcommand{\modu}[1]{\left \lvert #1 \right \rvert }

\DeclareMathOperator{\expec}{\mathbb{E}}

\definecolor{britishracinggreen}{rgb}{0.0, 0.26, 0.15}
\newcommand{\xcross}{\textcolor{red}{$\times$}}
\newcommand{\greenV}{\textcolor{britishracinggreen}{$\checkmark$}}

\newcommand{\eps}{\varepsilon}
\newcommand{\bo}{\mathcal{O}}

\newcommand{\Uni}{\mathcal{U}}
\newcommand{\diff}{\mathrm{d}}

\newtheorem{theorem}{Theorem}
\newtheorem{corollary}{Corollary}
\newtheorem{definition}{Definition}
\newtheorem{lemma}{Lemma}
\newtheorem{remark}{Remark}
\newtheorem{example}{Example}
\newtheorem{proposition}{Proposition}
\newtheorem{Conjecture}{Conjecture}

\DeclareMathOperator{\var}{\mathbb{V}}

\DeclareMathOperator{\ent}{\mathcal{H}}

\DeclareMathOperator{\supp}{Supp}

\DeclareMathOperator{\diag}{diag}

\DeclareMathOperator{\argmax}{argmax}

\newcommand{\pr}[1]{\left({#1}\right)}

\begin{document}

%

%

\twocolumn[

\aistatstitle{Compression with Exact Error Distribution for Federated Learning}

\aistatsauthor{ Mahmoud Hegazy \And Rémi Leluc \And Cheuk Ting Li \And Aymeric Dieuleveut  }

\aistatsaddress{\'Ecole Polytechnique \\ IPParis, France \And \'Ecole Polytechnique \\ IPParis, France \And Chinese University of Hong Kong \\ Hong Kong, China \And\'Ecole Polytechnique \\ IPParis, France}]
\begin{abstract}
Compression schemes have been extensively used in Federated Learning (FL) to reduce the communication cost of distributed learning. While most approaches rely on a bounded variance assumption of the noise produced by the compressor, this paper investigates the use of compression and aggregation schemes that produce a specific error distribution, \textit{e.g.}, Gaussian or Laplace, on the aggregated data. We present and analyze different aggregation schemes based on \textit{layered quantizers} achieving exact error distribution. We provide different methods to leverage the proposed compression schemes to obtain \textit{compression-for-free} in differential privacy applications. Our general compression methods can recover and improve standard FL schemes with Gaussian perturbations such as Langevin dynamics and randomized smoothing.

\end{abstract}


\section{INTRODUCTION}
\label{sec:1_intro}

Machine learning has become increasingly data-hungry, requiring vast amounts of data to train models effectively. Federated learning (FL) \citep{mcmahan2017communication,karimireddy2020scaffold,kairouz2021advances} has emerged as a promising approach for collaborative machine learning in distributed settings. In FL, multiple parties with their local datasets participate in a joint model training process without exchanging their raw data. However, communication from these devices to a central server can be slow and expensive resulting in a bottleneck. Thus, compression schemes have been widely used to reduce the size of the updates before transmission. Standard compression schemes \citep{konevcny2016federated,alistarh2017qsgd,wen2017terngrad,DBLP:conf/iclr/LinHM0D18,li2022communication} typically assume relatively bounded variance assumptions on the noise produced by the compressor. These schemes do not take into account the specific error distribution of the noise, which may depend on the input distribution. This can be hindering in settings where control of the noise shape can provide tighter analysis or can enable the fulfilment of additional constraints, \textit{e.g.} differential privacy. 

In a wider scope, different structures of the compression error have been extensively leveraged beyond FL such as in: lossy data compression for audio and music compression \citep{johnston1988transform,sayood2017introduction}; medical imaging \citep{liu2017current,ammah2019robust}; error control in wireless communications \citep{costello1998applications}; quantization in deep learning \citep{rastegari2016xnor} with weight pruning \citep{han2015learning,anwar2017structured}, low-precision weights \citep{jacob2018quantization,jung2019learning} and data perturbation for privacy-preserving machine learning \citep{abadi_deep_2016}. In this context, the main goal of this paper is to answer the following research questions: \textit{(i) Can we go beyond standard assumptions on the variance and ensure a particular (continuous) distribution of the compression error? (ii) What are the benefits, communication cost and applications of such schemes?}

\textbf{Related work.} Recently, different compression techniques with a precise noise distribution have been applied to various machine learning tasks. \citet{UniversallyQuantizedNN2022} apply subtractively dithered quantization \citep{dithering_Roberts1962, ditheringZiv1985,rate_distortion_Zamrir1992} to ensure a uniformly distributed noise in neural compression. Relative entropy coding \citep{havasi2019minimal,flamich2020compressing} aims at compressing the model weights and latent representations to a number of bits approximately given by their relative entropy from a reference distribution. Differential privacy can be achieved by ensuring a precise noise distribution, using lattice quantization \citep{amiri2021compressive, JointPrivacy_Lang23}, minimal random coding \citep{shah2022optimal}, or randomized encoding \citep{chaudhuri2022privacy}. These methods can be regarded as examples of channel simulation \citep{bennett2002entanglement,harsha2010communication,sfrl_trans}, which is a point-to-point compression scheme where the output follows a precise conditional distribution given the input. Refer to \citep{yang2023introduction,theis2022lossy} for more applications of channel simulation on neural compression.

In this paper, we not only consider a point-to-point compression setting, but also a distributed mean estimation and aggregation setting \citep{suresh2017distributed}, which is one of the most fundamental building blocks of FL algorithms \citep{kairouz2021advances}. There are $n$ users holding the data $x_1,\ldots,x_n$ respectively (which can be the gradient as in FedSGD, or the model weights as in FedAvg \citep{mcmahan2017communication}), who communicate with the server to allow the server to output an estimate $Y$ of the mean, with a precise noise distribution $Y-n^{-1}\sum_i x_i \sim Q$, which can then be used to perform model updates with a more precise behavior. 

There are two approaches to this problem, namely \emph{individual mechanisms} where point-to-point compression is performed between each user and the server, and the server simply averages the reconstructions of the data of each user (possibly with some postprocessing); and \emph{aggregate mechanisms} where the encoding functions of all users are designed as a whole to ensure a precise noise distribution of the final estimate $Y$.
For individual mechanisms, we study the communication costs of the direct and shifted layered quantizers based on \citep{wilson2000layered,hegazy2022randomized}. For aggregate mechanisms, we propose a novel method, called the \emph{aggregate Gaussian mechanism}, to ensure that the overall noise distribution $Q$ is exactly Gaussian, and analyze its communication cost.

In particular, for differential privacy (DP), the above schemes allow us to consider two trust settings. First, for a completely trusted server, if the clients wish to prevent the output of the server from leaking information on their data $x_1,\ldots,x_n$, a DP restriction \citep{dwork2006calibrating} can be imposed on the output $Y$. This is achieved by requiring the noise distribution $Q$ to be a privacy-preserving noise distribution. For example, a Gaussian noise can guarantee $(\varepsilon,\delta)$-DP \citep{dwork2014algorithmic} and  R\'enyi DP \citep{mironov2017renyi}. Second, when the server is \textit{less-trusted}, the clients wish for the server to not know their individual datapoints but trust it to faithfully carry out some postprocessing to make the output DP against external observers. In this case, it is vital to ensure that the compression mechanism is \emph{homomorphic}, and the messages sent from the users to the server can be aggregated before decoding, so that it is compatible with secure aggregation (SecAgg) techniques such as \citep{bonawitz2017practical} and other homomorphic cryptosystems. The aggregate Gaussian mechanism we propose is homomorphic, making it suitable for both privacy concerns of trusted and less-trusted servers.

\textbf{Contributions.} (1) We propose different quantized aggregation schemes based on \textit{layered quantizers} that produce a specific error distribution, \textit{e.g.} Gaussian or Laplace. (2) We provide theoretical guarantees and practical implementation of the developed aggregate mechanisms. (3) We exhibit FL applications in which we directly benefit from an exact Gaussian noise distribution, namely \textit{compression for free} with differential privacy and compression schemes for Langevin dynamics and randomized smoothing.

\textbf{Notation.} The term $\log$ refers to the logarithm in base $2$ while $\ln$ denotes the natural logarithm. The floor and ceil functions are denoted by $\lfloor \cdot \rfloor$ and $\lceil \cdot \rceil$ respectively with $\lceil x \rfloor := \lfloor x + 1/2 \rfloor$. For $n \in \nset$, we refer to $\{1,\ldots,n\}$ with the notation $[n]$. For a discrete probability distribution $p_X$ supported on $\mathcal{X}$ and random variable $X \sim p_X$, $\ent(X)$ denotes the entropy of $X$, \textit{i.e.} $\ent(X) = \ent(p_X)= - \sum_{x \in \mathcal{X}} p_X(x) \log p_X(x)$. $\ent(X|Y)$ denotes the conditional entropy of $X$ given $Y$, \textit{i.e.,} $\ent(X|Y) = \mathbb{E}[\ent(p_{X|Y}(\cdot | Y)]$. For a continuous probability distribution $f_X$ supported on $\mathcal{X}$ and random variable $X \sim f_X$, $h(X)$ denotes the differential entropy of $X$, \textit{i.e.} $h(X) = - \int_{\mathcal{X}} f_X(x) \log f_X(x) \mathrm{d}x$. The Lebesgue measure is denoted by $\lambda$ and $\mathcal{U}(a,b)$ with $a<b$ refers to the uniform distribution on $(a,b)$. All proofs, additional details and experiments are available in the appendix.
\vspace{-0.18cm}

\section{BACKGROUND, MOTIVATION}
\label{sec:quant_agg}

\textbf{Quantized aggregation.} Consider $n$ clients ($n \geq 1$). In the standard FL setting, the training process relies on locally generated randomness at individual participant devices. In this paper, we allow the clients and the server to have shared randomness. Practically, shared randomness can be generated by sharing a small random seed among the clients and the server, allowing them to generate a sequence of shared random numbers. We will see later that the existence of shared randomness can greatly simplify the schemes.

Let $S_i \in \mathcal{S}$ be the shared randomness between client $i$ and the server, and $T \in \mathcal{T}$ be the global shared randomness between all clients and the server. When building an FL algorithm, one is allowed to choose the joint distribution $P_{(S_i)_i,T}$ where these variables are usually (but not necessarily) taken to be mutually independent. Client $i\in [n]$ holds the data $x_{i}\in\mathbb{R}^d$ and for privacy and communication constraints, performs encoding to produce the description $M_{i}=\mathscr{E}(x_{i},S_{i},T) \in \mathcal{M}$, where $\mathscr{E}:\mathbb{R}^d\times\mathcal{S}\times\mathcal{T}\to\mathcal{M}$ is the \emph{encoding function}, and $\mathcal{M}$ is the set of descriptions (usually taken to be $\mathbb{Z}^d$). Given the descriptions $M_{1},\ldots,M_{n}$, the server produces the reconstruction $Y=\overline{\mathscr{D}}(M_{1},\ldots,M_{n},S_{1},\ldots,S_{n},T)$ which is an estimate of the average $n^{-1}\sum_{i=1}^{n}x_{i}$, where $\overline{\mathscr{D}}:\mathcal{M}^{n}\times\mathcal{S}^{n}\times\mathcal{T}\to\mathbb{R}^d$ is the \emph{overall decoding function}. The goal is to control the distribution of the quantization error as described in the next definition.

\begin{definition}(Aggregate AINQ mechanism) A quantization scheme with $n$ clients holding data $x_1,\ldots,x_n$ and a server producing $Y$ satisfies the Additive Independent Noise Quantization (\QD{}) property if the quantization error follows a target distribution $Q$ regardless of $\ens{x_i}_{i=1}^n$, i.e., 
\begin{equation}
    Y-\left(\frac{1}{n}\sum_{i=1}^{n}x_{i}\right) \sim Q.\label{eq:ainq_agg}
\end{equation} 
\end{definition}
A special case is when $n=1$, which we call \emph{point-to-point \QD{} mechanism}. In this case, an \QD{} mechanism with shared randomness  $S \in \scal$ can be performed in the following steps: (1) sample $S\sim P_S$, then (2) encode $M=\ecal(x,S)$ and (3) decode $Y=\dcal(M,S)$. We do not require the global shared randomness $T$ here since it plays the same role as $S$. We require the quantization noise to follow a particular distribution $Y-x \sim Q$. The simplest 
point-to-point \QD{} mechanism
is subtractively dithered quantization, which produces a uniform noise distribution.

\begin{example}(Subtractive Dithering)\label{ex:sub_dither} For a given step size $w>0$ and input $X$, subtractive dithering works by sampling $S\sim \Uni(-1/2,1/2)$, encoding the message $M=\lceil X/w+S \rfloor$, and decoding $Y=\open{M-S}w$. Then $Y-X\sim \Uni(-w/2;w/2)$ and is independent of $X$.
\end{example} 
An aggregate \QD{} mechanism for $n$ users can be constructed via a point-to-point \QD{} mechanism. For this approach to be applicable, the overall quantization noise must have a divisible distribution that can be expressed as a sum of $n$ \textit{i.i.d.} random variables.

\begin{definition}(Individual \QD{} mechanism)\label{def:ind_ainq}
An individual \QD{} mechanism is an aggregate \QD{} mechanism built via a point-to-point \QD{} mechanism where the overall quantization noise $Y - \sum_i x_i$ is divisible, $T$ is empty, the shared randomness is $S_1,\ldots,S_n$ which are i.i.d. copies of $S$, user $i$ produces $M_i = \ecal(x_i,S_i)$ and the server outputs $Y=n^{-1} \sum_i \dcal(M_i,S_i)$. 
\end{definition}

\textbf{Application 1: FL and Differential Privacy.}  The inherent sensitivity of individual data raises significant concerns about privacy breaches in the distributed setting. To address these challenges, the integration of differential privacy into federated learning has emerged as a compelling approach \citep{abadi_deep_2016,truex_hybrid_2019,wei2020federated,noble2022differentially}.

\begin{definition}\label{def:dp}(Differential-Privacy (DP)) Any algorithm $\mathcal{A}$ is $(\varepsilon,\delta)$-differentially private, if for all adjacent datasets $\mathcal{D}_1$ and $\mathcal{D}_2$ and all subsets $E \subset \mathrm{Im}(\mathcal{A})$ \vspace{-0.1cm}
\begin{equation} \label{eq:dp}
    \mathbb{P}(\mathcal{A}(\mathcal{D}_1) \in E) \leq e^{\varepsilon} \mathbb{P}(\mathcal{A}(\mathcal{D}_2) \in E) + \delta,
\end{equation} \vspace{-0.15cm}
where $\mathbb{P}$ is over the randomness used by algorithm $\mathcal{A}$.
\end{definition}
The Gaussian mechanism injects controlled noise into computations, allowing for a balance between privacy protection and data utility. For a function $f$ that operates on a dataset $\mathcal D$, it is defined as \vspace{-0.1cm}
\begin{equation} \label{eq:gauss_mech}
\mathrm{G}(\mathcal D) = f(\mathcal D)+\mathcal{N}(0,\sigma^2\mathrm{I}) \vspace{-0.1cm}
\end{equation}
and it is guaranteed to be $(\eps, \delta)$-DP if the noise satisfies $\sigma^2 \geq 2\Delta^2_2 \ln{(1.25/\delta})/\eps^2$ \citep{dwork2014algorithmic} where $\Delta_2 = \sup_{\mathcal{D}_1, \mathcal{D}_2} \norm{f(\mathcal{D}_1)-f(\mathcal{D}_2)}_2$ for $\mathcal{D}_1$ and $\mathcal{D}_2$ differing on one element. While common privacy-preserving approaches rely on adding a Gaussian or Laplace noise on top of compression schemes, one can leverage AINQ mechanisms to directly obtain privacy guarantees with a reduced communication cost. For example, setting the compression error to be a properly scaled Gaussian recovers the Gaussian mechanism of Eq.\eqref{eq:gauss_mech}. 

\textbf{Application 2: FL and Langevin dynamics.} When solving the Bayesian inference problem in the FL setting \citep{vono2022qlsd}, compression operators $\mathscr{C}:\rset^d \to \rset^d$ are commonly unbiased and have a bounded variance. For a loss function $H$ deriving from a potential, the stochastic Langevin dynamics starts from $\theta_0 \in \rset^d$ and is updated as 
\begin{equation*}
\theta_{k+1} = \theta_{k} - \gamma H(\theta_k) + \sqrt{2 \gamma} Z_{k+1},    
\end{equation*}
with $Z_k \sim \mathcal{N}_d(0,\mathrm{I}_d)$ and $\gamma >0$. Using compression with exact error distribution, one can precisely control the distribution of the error $(\mathscr{C}(X) - X)$ and recover the QLSD scheme of \citet{vono2022qlsd} with a reduced communication cost. This may be done via the quantizer $\mathscr{C}_\gamma$ such that $\mathscr{C}_{\gamma}(X) - X \sim \mathcal{N}_d(0,2\mathrm{I}_d/\gamma)$ along with the update rule $\theta_{k+1} = \theta_{k} - \gamma \mathscr{C}_{\gamma}(H(\theta_k))$. (See Appendix \ref{app:langevin} for more details) 

\textbf{Application 3: FL and Randomized Smoothing.} Fast rates for non-smooth optimization problems of the form $\min_{\theta \in \rset^d} \{f(\theta)=\sum_{i=1}^n f_i(\theta)\}$ can be attained using the smoothing approach of \citet{duchi_randomized_2012,scaman_optimal_2018}. These accelerated algorithms such as Distributed Randomized Smoothing (DRS) rely on a \textit{smoothed} version $f_\sigma$ of $f$ defined by 
\begin{equation*}
    f_{\sigma}(\theta) = \expec_\xi[f(\theta+\sigma \xi)],
\end{equation*}
where $\xi \sim \mathcal{N}(0,\mathrm{I}_d)$ and $\sigma>0$. Each client approximate the smoothed gradient with a subgradient $g_i$ evaluated at $m$ perturbed points $g_i(\theta + \sigma \xi_j)$ for $j \in [m]$. Interestingly, the sampling steps may be replaced with compressors that produce exact error distribution. In the spirit of \citet{philippenko2021preserved}, one can first compress the model parameter $\theta$ with a Gaussian error distribution as $\ecal(\theta) = \theta + \sigma \xi$ and then evaluate the subgradients at compressed point as $g_i(\ecal(\theta))$ to recover the classical DRS algorithm.

\section{INDIVIDUAL MECHANISMS} \label{sec:indiv}

In order to obtain a target noise distribution exactly, we describe two point-to-point \QD{} mechanisms based on \citet{hegazy2022randomized} and on the layered multishift coupling in \citet{wilson2000layered}, which can be used to construct $n$-client individual AINQ mechanisms using Def. \ref{def:ind_ainq}. For geometric interpretations, we refer to them respectively as the \textit{direct layered quantizer} and the \textit{shifted layered quantizer}. They both rely on subtractive dithering with step size $w$ to generate an error which follows a uniform distribution $\Uni(-w/2,w/2)$, as in Example \ref{ex:sub_dither}, but with a random step size $w$.

\subsection{Individual Mechanisms via Direct and Shifted Layered Quantizers}\label{sec:layered}

Consider a random variable $Z$ following a unimodal distribution $f_Z$. Instead of having a deterministic value of $w$, it is randomly sampled from an appropriate distribution ensuring that the marginal error follows $f_Z$. Define the superlevel set $\mathcal{L}_x(f_Z) \coloneqq \{u \in \rset: f_Z(u) \geq x \}$ for any $x \in \rset$ and let $\bar Z \coloneqq \max f_Z$.  For any $x\in \open{0,\bar Z}$, define $b_Z^-(x) = \inf \mathcal{L}_{x}(f_Z)$ and $ b_Z^+(x) = \sup \mathcal{L}_{x}(f_Z)$. 
As illustrated in Figure \ref{fig:shifted_vs_direct}, for a unimodal distribution, sampling under the graph can be thought of as sampling from a infinite mixture of uniform distributions over continuous intervals.

\textbf{Direct Layered Quantizer.} In order to sample from a real continuous random variable, it is sufficient to uniformly sample in the area under its graph and then project onto the $x$-axis. The construction below has been mentioned in the special case of Gaussian noise in \citet{UniversallyQuantizedNN2022}, and studied in the general unimodal case in \citet{hegazy2022randomized}.

\begin{definition}(Direct layered quantizer)
\label{def:direct}
Given any $Z$ with a unimodal p.d.f. $f_Z$, the direct layered quantizer, producing an error $Z$, is defined using the encoder $\ecal$, the decoder $\dcal$, and $P_S$ such that $S=(U,D_Z)$ with $U \sim \Uni(0,1)$ independent of $D_Z \sim f_{D}$ where the p.d.f. $f_{D}$ is defined through the superlevel sets of $f_Z$. For all $x \in \rset$, $f_{D}(x) = \lambda(\mathcal{L}_x(f_Z))=\open{b_Z^+(x)-b_Z^-(x)}\one_{ (0,\bar{Z})}(x)$ and 
\begin{align*}
    &M= \ecal(X,S) = \lceil X/f_{D}(D_Z) + U \rfloor,\\
    &\dcal(M,S) = \open{M-U} f_{D}(D_Z) + \frac{b_Z^+(D_Z)+b_Z^-(D_Z)}{2} \cdot
\end{align*}
\end{definition}

 \begin{figure}[t]
     \centering
     \includegraphics[width=8cm]{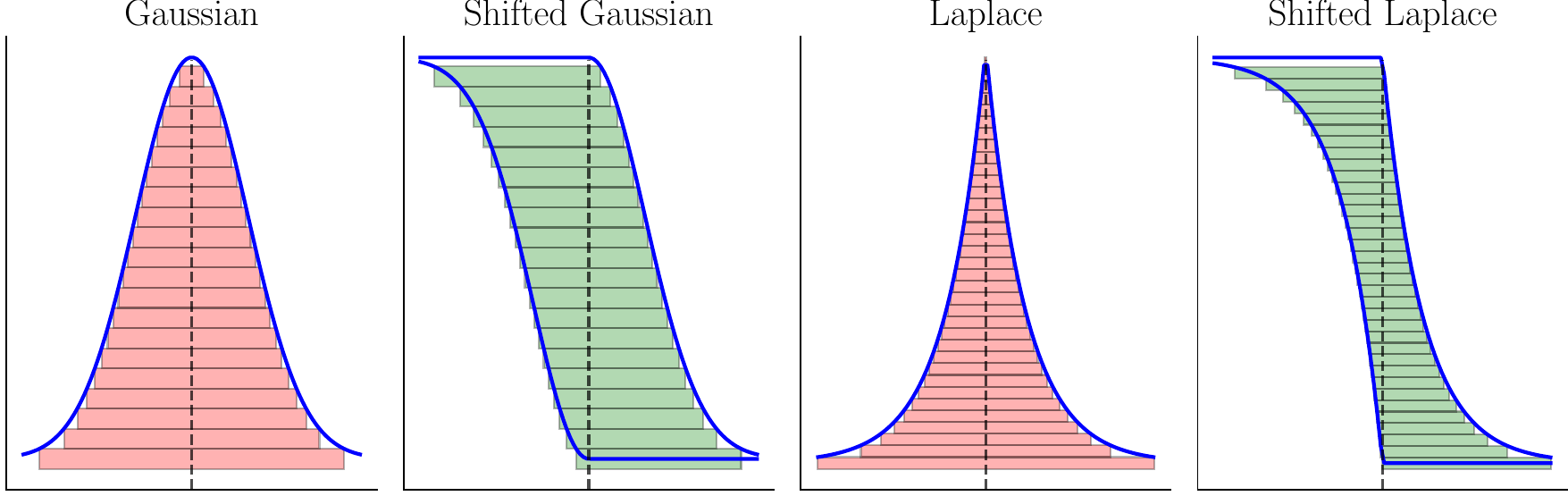}
     \caption{Illustration of the sampling area for different distribution: (shifted) Gaussian and (shifted) Laplace.}
     \label{fig:shifted_vs_direct}
     \vspace{-0.1cm}
 \end{figure}

\textbf{Shifted Layered Quantizer.}
Another approach for leveraging subtractive dithering 
is based on multishift coupling, as described in \cite{wilson2000layered}. It is based on the idea of using a sequence of shifted uniform distributions to generate a sequence of target distributions. The key idea is that the sampling from an area $A$ and then projecting onto the $x$-axis is equivalent to the sampling from any horizontal reflection of $A$ and then projecting onto the $x$-axis. Thus, one can can flip one side of a unimodal distribution and still take samples from the part under the flipped side (see Figure \ref{fig:shifted_vs_direct} with the shifted distributions).
\begin{definition} 
    \label{def:shifted}
    (Shifted layered quantizer) Given any $Z$ with a unimodal p.d.f. $f_Z$, the layered randomized quantizer, producing an error $Z$, is defined using the encoder $\ecal$, the decoder $\dcal$ and randomness $P_S$ such that $S=(U,W_Z)$ with $U \sim \Uni(0,1)$ independent of $W_Z \sim f_{W}$ where the p.d.f. $f_{W}$ is defined through the superlevel sets of $f_Z$. For all $x \in \rset$ ,
\begin{align*}\label{eq:shifted_layered_eq}
&f_{W}(x) = \open{b_Z^+(x)-b_Z^-(\bar Z-x)}\one_{\open{0,\bar Z}}(x), \\
&\ecal(X,S) = \lceil X/ f_W(W_Z) + U \rfloor,\\
&\dcal(M,S) =\open{M\! -\! U} f_W(W_Z) + \frac{b_Z^+(W_Z)+b_Z^-(\bar Z \! -\! W_Z)}{2} .
\end{align*}
\end{definition}

By construction, the shifted layered quantizer satisfies the \QD{} property (see details in Appendix \ref{subsec:equiv_quantizers}).

\subsection{Communication Complexity}

To transmit the integer $M$, we should encode it into bits. There are two general approaches. First, if a fixed-length code is used where $M$ is always encoded to the same number of bits, then $\lceil \log |\supp M| \rceil$ bits are required.
Second, if a variable-length code is used, for example, if we encode $M$ using the Huffman code \citep{huffman1952method} on the conditional distribution $p_{M|S}$, then we require an expected encoding length bounded between $\ent(M|S)$ and $\ent(M|S)+1$.
We will see that the direct layered quantizer gives a better variable-length performance, whereas the shifted layered quantizer gives a better fixed-length performance.

For variable-length codes, it has been shown in \citep[Theorem~4]{hegazy2022randomized} that every \QD{} scheme with error distribution $f_Z$ and uniform input $X\sim \Uni(0,t)$ must satisfy 
\begin{equation} \label{eq:univ_lower2}
      \ent(M|S) \geq \log(t)+h(D_Z),
\end{equation}
where $D_Z$ is defined in Definition \ref{def:direct}, and $-h(D_Z)$ is called the \emph{layered entropy}. It has also been shown in \citet{hegazy2022randomized} that the direct layered quantizer is almost optimal, in the sense that as long as $Z$ has a unimodal distribution and $\supp(X) \subseteq \closed{0,t}$, it achieves
\begin{equation}\label{eq:upper_direct2}
    \ent(M|S) \leq \log(t)+\frac{8\log{(e)}}{t}\sqrt{\var\closed{Z}}+h(D_Z).
\end{equation}
The gap between \eqref{eq:univ_lower2} and \eqref{eq:upper_direct2} tends to $0$ as $t \to \infty$.

While the shifted layered quantizer is not asymptotically optimal, the optimality gap (the gap between the $\ent(M|S)$ attained and its optimal value) is still relatively small as shown below.

\begin{proposition} (Optimality Gap) \label{prop:gap2} For a target unimodal noise distribution $f_Z$ that is symmetric, the shifted layered quantizer achieves
\begin{equation*} \label{eq:upper_shifted}
 \ent(M|S) \leq \log(t)+\frac{8\log(e)}{t}\sqrt{\var\closed{Z}}+h(W_Z),
\end{equation*}
and the optimality gap of using the shifted layered quantizer is upper bounded by $(8\log(e)/t)\sqrt{\var\closed{Z}} + 2$.
\end{proposition}

Figure \ref{fig:layered_quant_compare} below shows the conditional entropy $\ent(M|S)$ needed to simulate a Gaussian noise and a Laplace noise with standard deviation $\sigma \in \{1;3\}$ according to the support size $t$ where $X\sim \Uni(0,t)$. The gap is smaller than $1$ bit for all the values computed.

The advantage of the shifted layered quantizer is its fixed-length performance. A fixed-length code has the advantage that we do not have to build a Huffman code on the conditional distribution $p_{M|S}$ for each $S$, which may be infeasible. Since the shifted layered quantizer has a quantization step size bounded away from $0$ as shown in Figure \ref{fig:shifted_vs_direct}, it provides a fixed upper bound on the number of quantization bits.

 \begin{figure}[h!]
  \centering
  \subfigure[Gaussian error]{
  \includegraphics[scale=0.27]{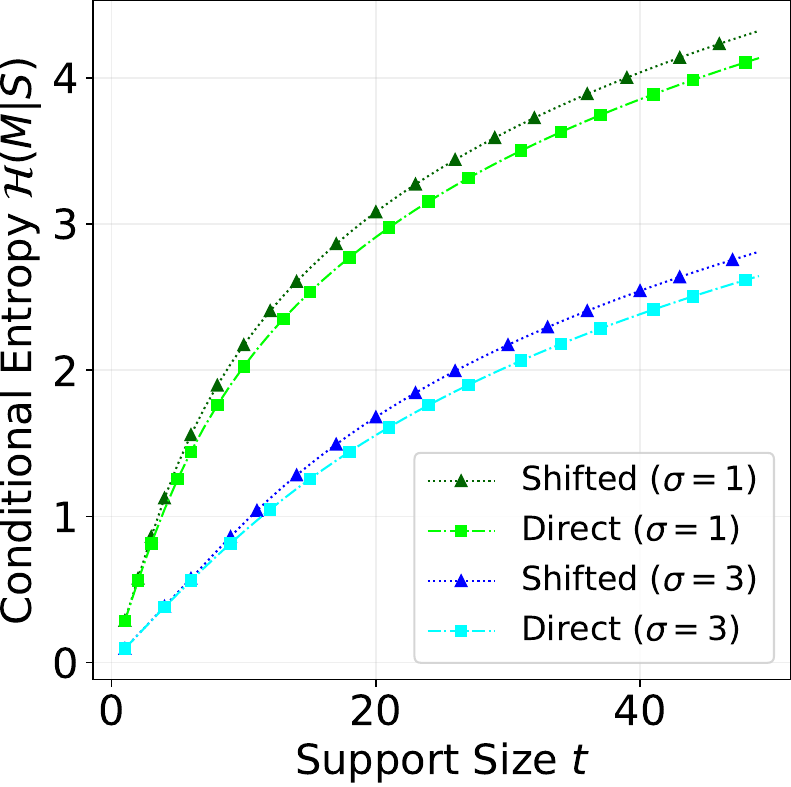}\label{fig:gaus_cond_entropy}}
  \subfigure[Laplace error]{
  \includegraphics[scale=0.27]{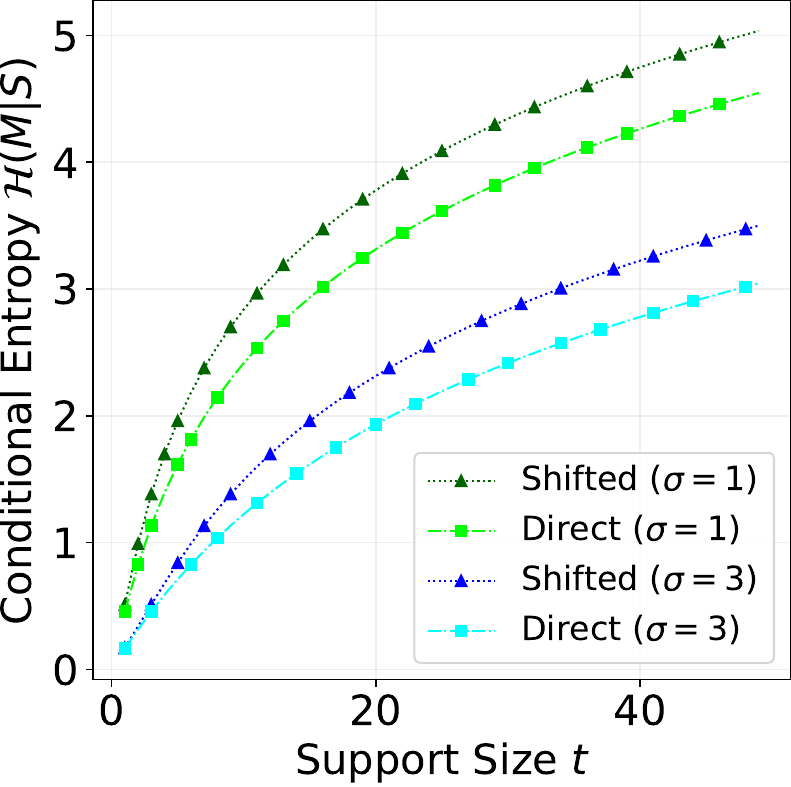}\label{fig:lap_cond_entropy}}
    \caption{\label{fig:layered_quant_compare}Conditional entropy $\mathcal{H}(M|S)$ of the layered quantizers with Gaussian/Laplace error distribution.}
    \label{fig:entropies}
\end{figure}

\begin{proposition}(Minimal step size for shifted layered) \label{prop:min_step2}
For a required $f_Z$, denote by $\eta_Z$ the minimal step size of the shifted layer quantizer, i.e. $\eta_Z\coloneqq \min f_W$. Assume that the input $X$ lies in a fixed interval of length $t$. \\
{\rm \textbullet} \ If $Z \sim \text{Laplace}(0,\sigma/\sqrt{2})$ then $\eta_Z = \sigma \sqrt{2} \ln{2}$ and $\modu{\supp M}\leq 2 + t/(\sigma \sqrt{2} \ln{2})$. \\ 
{\rm \textbullet} \ If $Z \sim \mathcal{N}(0,\sigma^2)$ then $\eta _Z= 2 \sigma  \sqrt{\ln 4}$ and \\ $\modu{\supp M}\leq 2 + t/(2 \sigma  \sqrt{\ln 4})$.
\end{proposition}

\section{AGGREGATE MECHANISMS} \label{sec:agg_mec}

\subsection{Homomorphic Mechanisms}

For an individual \QD{} mechanism, the descriptions $M_1,\ldots,M_n$ sent from the users must all be available to the server. The descriptions cannot be first aggregated into a single number before decoding. We study aggregate \QD{} mechanisms where such ``intermediate aggregation'' of descriptions is possible.

\begin{definition}(Homomorphic) An aggregate \QD{} mechanism is \emph{homomorphic} if there exists functions $\mathscr{D}:\mathcal{M}\times\mathcal{S}\times\mathcal{T}\to\mathbb{R}^d$
(called the \emph{homomorphic decoding function})
such that the overall decoding function is \vspace{-0.2cm}
\begin{equation} \label{eq:hom}
\overline{\mathscr{D}}(m_{1},s_{1},\ldots,m_{n},s_{n},t)=\frac{1}{n}\mathscr{D}\left(\sum_{i=1}^{n}m_{i},\,\sum_{i=1}^{n}s_{i},\,t\right)    
\end{equation}
and for all $t$, $\mathscr{D}(\cdot,\cdot,t):\mathcal{M}\times\mathcal{S}\to\mathbb{R}^d$
is a homomorphism, i.e., $\mathscr{D}(m,s,t)+\mathscr{D}(m',s',t)=\mathscr{D}(m+m',s+s',t)$ for all $m,m',s,s'$. 
\end{definition}

Here we require $\mathcal{M}$ and $\mathcal{S}$ to be abelian groups (e.g., $\mathcal{M} = \mathbb{Z}^d$, $\mathcal{S}=\mathbb{R}^d$) so we can perform addition over $\mathcal{M}$ and $\mathcal{S}$. For a homomorphic scheme, the server only requires the common randomness $(S_{i})_{i},T$ and
the sum of the descriptions $\sum_{i=1}^{n}M_{i}$, not the individual
descriptions $M_{1},\ldots,M_{n}$. 

In federated learning, homomorphic mechanisms have several
advantages. First, the descriptions $M_{1},\ldots,M_{n}$ can be passed
from the clients to the server through a network for sum computation \citep{ramamoorthy2013communicating,rai2012network,qu2021decentralized,rizk2021graph}, resulting in a reduction of communication cost.
A simple method is to have each node in the network add all incoming
signals and pass the sum to the next node. The internal nodes in the
network do not need to access $(S_{i})_{i},T$. Second, homomorphic
mechanisms are compatible with SecAgg \citep{bonawitz2017practical} and other additive homomorphic cryptosystems. We can perform
SecAgg on $M_{1},\ldots,M_{n}$ so that the server can only know $\sum_{i=1}^{n}M_{i}$, but not the individual $M_{i}$'s so as to preserve privacy. 

The layered quantizers of the previous section do not produce homomorphic mechanisms. The reason is that the quantization step sizes of the $n$ users are randomly drawn and different. Thus, the quantized data $M_i$'s are at different scales and cannot be added together. A simple and naive way to obtain an homomorphic mechanism is to rely on subtractive dithering (Example \ref{ex:sub_dither}) with a fixed step size $w$ for each user. 
\vspace{-0.1cm}
\subsection{Irwin-Hall Mechanism}\label{subsec:irwinhall}

When applying subtractive dithering for each of the $n$ users, the resultant noise of the individual \QD{} mechanism follows a scaled Irwin-Hall distribution. More precisely, consider $S=(S_{1},\ldots,S_{n})$ to be an \textit{i.i.d.} sequence of $\mathcal{U}(-1/2,1/2)$ random variables and $T=0$ to be degenerate. The encoding function is $M_i = \mathscr{E}(x_i,S_i)=\lceil x_i/w+S_i \rfloor$ where $w:=2\sigma\sqrt{3n}$, and the decoding function is $Y =w( \sum_i M_i - \sum_i S_i)$. The noise is a scaled Irwin-Hall distribution $\mathrm{IH}(n,0,\sigma^{2})$, where $\mathrm{IH}(n,\mu,\sigma^{2})$
denotes the distribution of $n^{-1}\sum_{i=1}^{n}Z_{i}+\mu$ with
$Z_{1},\ldots,Z_{n}\stackrel{iid}{\sim}\mathcal{U}(-\sigma\sqrt{3n},\,\sigma\sqrt{3n})$.
Note that $\mathrm{IH}(n,0,\sigma^{2})$ has mean $0$ and variance
$\sigma^{2}$, and approximates the Gaussian distribution $\mathcal{N}(0,\sigma^{2})$ when $n$ is large. We call this the \emph{Irwin-Hall mechanism}. 

While this mechanism is simple and homomorphic (the decoding function only depends on $\sum_i M_i$ and $\sum_i S_i$), an obvious downside is that the resultant noise distribution is the Irwin-Hall distribution, not the Gaussian
distribution. The Irwin-Hall distribution itself is not a privacy-preserving
noise for $(\varepsilon,\delta)$-differential privacy nor R\'enyi privacy.
Moreover, specific FL applications such as stochastic Langevin dynamics \citep{welling2011bayesian,vono2022qlsd} or randomized smoothing \citep{duchi_randomized_2012,scaman_optimal_2018} specifically require a Gaussian distribution. This motivates the development of advanced aggregate AINQ mechanisms. 

\subsection{Aggregate $Q$ Mechanism}
\label{subsec:aggregate_mechanism}

Let $P=\mathrm{IH}(n,0,\sigma^{2})$ be the Irwin-Hall distribution. We have seen in Section~\ref{subsec:irwinhall} that the Irwin-Hall mechanism produces a noise with distribution $P$.
The idea is now to decompose the desired noise distribution $Q$ (e.g. Gaussian) into a mixture of shifted and scaled versions of $P$ in the form ``$aP + b$'', use the global common randomness $T$ to select the shifting and scaling factors according to the mixture probabilities, and perform the Irwin-Hall mechanism with the input and output shifted and scaled. In this section, we construct a homomorphic aggregate \QD{} mechanism with a noise distribution $Q$, called the \emph{aggregate $Q$ mechanism}, using this strategy. 

\begin{definition}(Mixture set)\label{def:mix_set}
For probability distributions $P,Q$, denote by $\Pi_{A,B}(P,Q)$ the set of joint probability distributions $\pi_{A,B}$ of the random variables $A\in\mathbb{R}$ and $B\in\mathbb{R}$ such that if $(A,B)\sim \pi_{A,B}$ is independent of $Z\sim P$, then $AZ+B\sim Q$.
\end{definition}

\begin{definition}(Aggregate $Q$ mechanism) \label{def:decomp_mec}
Let $S_{1},\ldots,S_{n}\stackrel{iid}{\sim}\mathcal{U}(-1/2,1/2)$ and $T=(A,B) \sim \pi_{A,B} \in \Pi_{A,B}(P,Q)$. The \emph{aggregate $Q$ mechanism} is defined by $w:=2\sigma\sqrt{3n}$ and
\begin{align*}
&\mathscr{E}(x,s,a,b):=\left\lceil x/(a w)+s \right\rfloor, \\
&\overline{\mathscr{D}}((m_i)_i,(s_i)_i,a,b):= \frac{a w}{n}\left(\sum_{i=1}^n m_i- \sum_{i=1}^n s_i\right) + b.
\end{align*}
\end{definition}
Basically, we generate $A,B$ randomly, and then run the Irwin-Hall mechanism scaled by $A$ and shifted by $B$. The resultant noise distribution is $Q$ as precised in the next Proposition.
\begin{proposition}\label{prop:decomp_ainq}
 The aggregate mechanism of Def.\ref{def:decomp_mec} satisfies the \QD{} property with noise distribution $Q$, and is homomorphic.\footnote{Technically, $\overline{\mathscr{D}}$ is not in the form (\ref{eq:hom}) due to the ``$+b$'' term, though it can be absorbed into $s_i$. Let $b'_i = b/n$, and treat $(s_i,b'_i)$ as the common randomness between client $i$ and the server. We then have $\overline{\mathscr{D}}((m_i)_i,(s_i,b'_i)_i,a)= \frac{a w}{n}\left(\sum_{i=1}^n m_i- \sum_{i=1}^n (s_i-b'_i)\right)$, which is in the form (\ref{eq:hom}) by taking $\mathscr{D}(m,s,b',a)=aw(m-s+b')$.}
\end{proposition}

\subsection{Aggregate Gaussian Mechanism}

We now study the case where $Q$ is the Gaussian distribution, and describe the construction of $\pi_{A,B}$ which decomposes $Q$ into a mixture of Irwin-Hall distributions. We call this the \emph{aggregate Gaussian mechanism}.

\textbf{Step 1.} First, we study how to decompose $\mathcal{U}(-1/2,1/2)$ (instead of $Q$) into a mixture of shifted and scaled versions of the Irwin-Hall distribution $P$ with a pdf $f$. Assume $f$ is appropriately scaled so its support is $[-1/2,1/2]$. For $-1/2\le x\le 1/2$, decompose $\mathcal{U}(-1/2,1/2)$ into the mixture $(1/f(0))f(x) + (1-1/f(0))\varphi(x)$ where $\varphi(x)=(f(0)-f(x))/(f(0)-1)$ is a bimodal distribution with modes $\{-1/2;1/2\}$, and can be decomposed into uniform distributions using the strategy in Section~\ref{sec:layered}. These uniform distributions can be recursively decomposed into a mixture of a shifted/scaled $f$ and other uniform distributions, and so on. 
See Figure \ref{fig:agg_gauss_fig}, and Algorithm $\textsc{DecomposeUnif}$ in Appendix \ref{subsec:decomp_uni}.

\textbf{Step 2.} To decompose the Gaussian distribution $Q$ with a pdf $g$, we first decompose $g$ into the mixture $g(x)=\lambda f(x)+(1-\lambda)\psi(x)$ where $\lambda$ is as large as possible such that $\psi$ is still unimodal. A practical choice is $\lambda=\inf_{x>0} \mathrm{d}g(x)/\mathrm{d}f(x)$ if $n\ge3$, and $\lambda=0$ if $n\le2$. We then decompose $\psi$ into a mixture of uniform distributions, using the strategy in Section~\ref{sec:layered}, and decompose those uniform distributions using the aforementioned strategy. Refer to Figure \ref{fig:agg_gauss_fig}, and Algorithm $\textsc{Decompose}$ in Appendix \ref{subsec:decomp}. 

\begin{figure}[h!]
\begin{centering}
\includegraphics[scale=0.6]{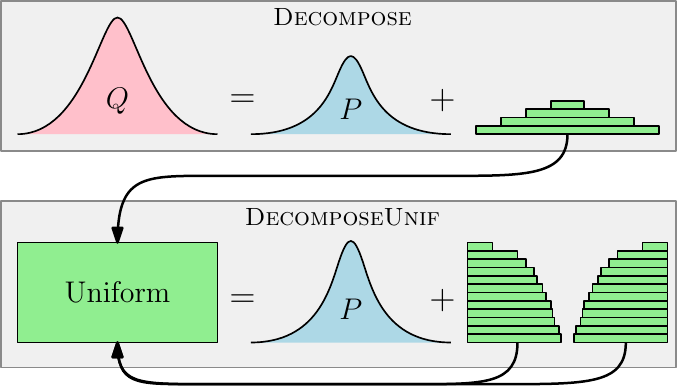}
\par\end{centering}
\caption{\label{fig:agg_gauss_fig} Decomposition of $Q$ into a mixture of scaled Irwin-Hall $P$. Algorithm $\textsc{Decompose}$ decomposes $Q$ into a mixture of $P$ and a unimodal distribution, where the latter is then decomposed into a mixture of uniform distributions and passed to $\textsc{DecomposeUnif}$. }
\vspace{-0.5cm}
\end{figure}

\subsection{Theoretical Analysis}
We now study the communication cost. Assuming the data $x_{i}$ is bounded by $|x_{i}|\le t/2$ for $i=1,\ldots,n$, we have $|\mathscr{E}(x_{i},S_{i},A,B)| \le\left\lceil t/(2w|A|)\right\rceil$. Therefore, $M_{i}$ can be encoded into $\lceil\log(t/(w|A|)+3)\rceil$
bits conditional on $A$. The expected amount of communication per client is upper-bounded by $\mathbb{E}(\lceil \log(t/(w|A|) +3)\rceil)$ which is approximately equal to $\mathbb{E}[-\log|A|]+\log(t/(2\sigma\sqrt{3n}))$ bits.
Hence, we can construct an aggregate \QD{} mechanism with a small amount of communication if we have a small $\mathbb{E}[-\log|A|]$. To study this quantity, we introduce a notion called relative mixture entropy.

\begin{definition}
[Relative mixture entropy] Given probability distributions $P,Q$
over $\mathbb{R}$, the \emph{relative mixture entropy} is $h_{\mathrm{M}}(Q\Vert P):=\sup_{\Pi_{A,B}(P,Q)}\mathbb{E}_{\Pi_{A,B}}[\log|A|]$ where the supremum is over $\Pi_{A,B}(P,Q)$.\footnote{Take $\log0=-\infty$, so $\mathbb{E}[\log|A|]=-\infty$ if
$\mathbb{P}(A=0)>0$.}
\end{definition}

Relative mixture entropy has several desirable properties similar to the differential entropy (see Appendix~\ref{app:relmixent}). In particular, $h_{\mathrm{M}}(Q\Vert P)\le h(Q)-h(P)$ can be upper-bounded in terms of the difference of differential entropies. We have the following bound on the communication cost in terms of $h_{\mathrm{M}}(Q\Vert P)$.

\begin{theorem}
\label{thm:irwin_bound}(Complexity) Let $P=\mathrm{IH}(n,0,\sigma^{2})$ be an
 Irwin-Hall distribution. Assume $|x_{i}|\le t/2$ for $i=1,\ldots,n$.
There exists an aggregate \QD{} mechanism for simulating a noise distribution $Q$, with an expected amount of communication per client upper-bounded by 
\[
-h_{\mathrm{M}}(Q\Vert P)+\log\frac{t}{2\sigma\sqrt{3n}}+\frac{6\sigma\sqrt{3n}\log e}{t}\cdot\frac{\mathbb{E}_{Z\sim Q}[|Z|]}{\mathbb{E}_{Z\sim P}[|Z|]}+1.
\]
\end{theorem}

To give an upper bound on the expected communication cost, it remains to give a lower bound on $h_{\mathrm{M}}(Q\Vert P)$.
\begin{theorem}
\label{thm:hM_symm}(Lower bound) For two distributions $P,Q$ with pdfs $f,g$, respectively, that are unimodal, differentiable and symmetric around $0$ with  $L:=2\sup\{x:f(x)>0\}<\infty$ and $\lambda:=\inf_{x>0}\mathrm{d}g(x)/\mathrm{d}f(x)$, we have
\begin{align*}
  h_{\mathrm{M}}(Q\Vert P) \ge-(1 \! - \! \lambda) \! \left(Lf(0)+\log\frac{eL(g(0)-\lambda f(0))}{2(1-\lambda)}\right)\!.
\end{align*}
\end{theorem}

\begin{figure}[h!]
\centering
\subfigure[$t=64$]{
  \includegraphics[scale=0.335]{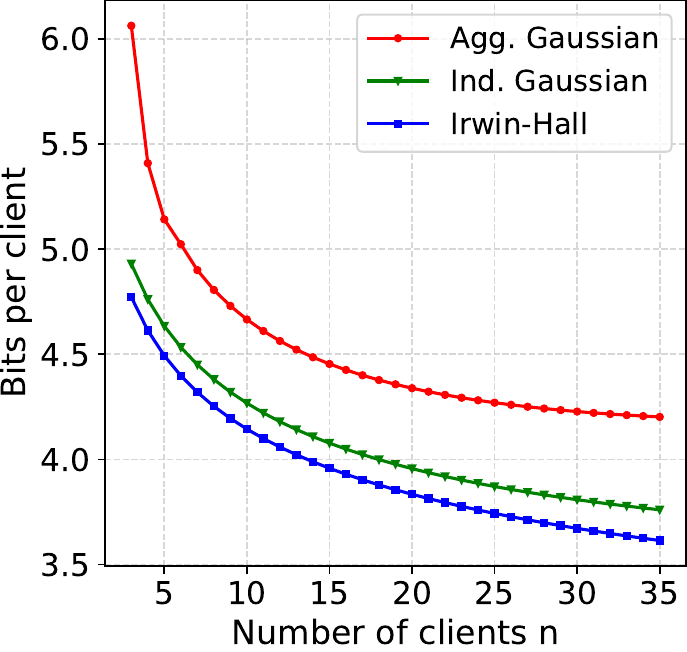}\label{fig:FL_agg_64}}
  \subfigure[$t=2048$]{
  \includegraphics[scale=0.335]{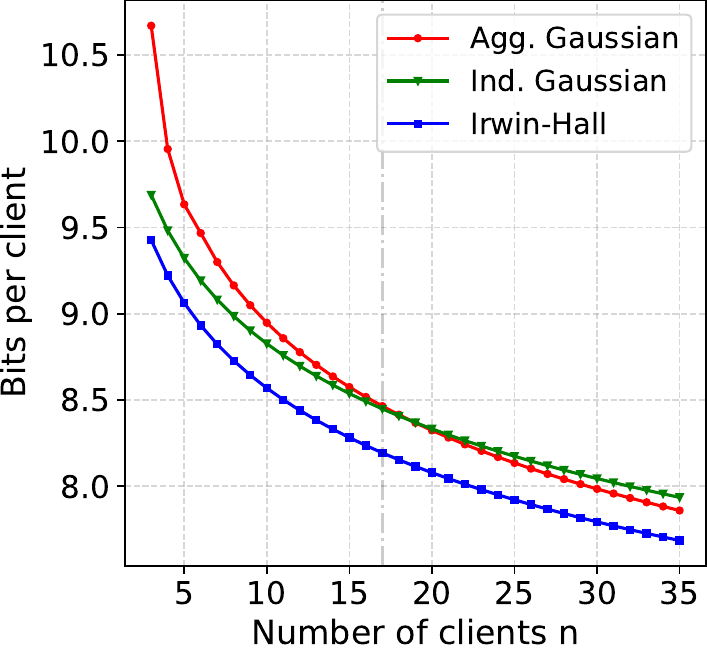}\label{fig:FL_agg_2048}}
\caption{\label{fig:agg_compare}Comparison between the aggregated Gaussian, individual Gaussian (direct) and Irwin-Hall mechanisms, $\sigma = 1$, (a) $x_{i}\in[-2^5;2^5]$, (b) $x_{i}\in[-2^{10};2^{10}]$.} 
\vspace{-0.5cm}
\end{figure}

We can then combine Theorems \ref{thm:irwin_bound}, \ref{thm:hM_symm} to give a bound on the communication cost of the aggregate Gaussian mechanism. 
The proofs of the theorems are given in Appendix \ref{subsec:pf_irwin_bound} and \ref{subsec:pf_hM_symm}.
The pseudocode is given in Algorithms $\textsc{Encode}$ and $\textsc{Decode}$ in Appendix \ref{subsec:agg_gauss_algo}. 
Figure~\ref{fig:agg_compare} compares the communication cost per client of the aggregate Gaussian mechanism, individual Gaussian mechanism (direct layered quantizer) and Irwin-Hall mechanism where the bounds are computed using Theorem \ref{thm:irwin_bound} and Theorem \ref{thm:hM_symm}. We can see that the aggregate Gaussian mechanism can have a smaller communication cost than the individual Gaussian mechanism for a large number of clients, though not as small as the Irwin-Hall mechanism. We also remark that aggregate Gaussian is homomorphic (unlike individual Gaussian), and has a noise distribution that is exactly Gaussian (unlike the Irwin-Hall mechanism).

\section{COMPRESSION AND SUBSAMPLING FOR  PRIVACY}\label{sec:5_dP_with_comp}

\subsection{Trusted server with subsampling strategy}

Subsampling \citep{li2012sampling,balle2018privacy} is a technique for enhancing privacy, where only a subset of clients or their data is selected for transmission. It can be leveraged to reduce communication cost while improving DP guarantees. Recently, \citep{chen2023privacy} introduced coordinate-wise subsampling to derive DP schemes with optimal communication utility tradeoffs. We now describe how this scheme can be improved with an AINQ mechanism. 

We are interested in the \emph{subsampled individual Gaussian mechanism (SIGM)} with shifted layered quantizer. First consider the case in dimension $d=1$ with $n$ clients and sampling variables $B_1,\ldots ,B_n \sim \mathcal{B}(\gamma)$  ($\gamma \in [0,1]$), where $B_i=1$ if client $i$ is selected and $B_i=0$ otherwise. Let $\tilde{n}=\sum_i B_i$ be the number of selected clients and let $\mathscr{E},\mathscr{D}$ be the encoding and decoding functions in Def.\ref{def:shifted} for the noise distribution $\mathcal{N}(0,(\sigma \gamma n)^2)$. The mechanism works as follows: \textit{(1)} generate the global shared randomness $B_1,\ldots ,B_n \sim \mathcal{B}(\gamma)$; \textit{(2)} generate the shared randomness $S_i$ if $B_i=1$; \textit{(3)} client $i$ sends $M_i = \mathscr{E}(x_i\sqrt{\tilde{n}}, S_i)$ if $B_i=1$, or $M_i=0$ if $B_i=0$; and \textit{(4)} server outputs $Y = (\gamma n\sqrt{\tilde{n}})^{-1}\sum_{i:\,B_i=1} \mathscr{D}(M_i,S_i)$. It can be checked that $Y - (\gamma n)^{-1}\sum_{i: B_i=1} x_i \sim \mathcal{N}(0,\sigma^2)$. The proof, the generalization to $d$ dimensional data, and the algorithm are included in Appendix \ref{subsec:sigm_alg}. This mechanism achieves the same statistical behavior as CGSM \citep{chen2023privacy}, with the benefit that it directly ensures an exact Gaussian noise through quantization without the need to first incur a compression error independently of DP noise. Invoking \citep[Theorem 4.1]{chen2023privacy} and Prop.~\ref{prop:min_step2}, we have the following result.

\begin{proposition}\label{prop:dp_sigm} If $x_1,\ldots,x_n \in [-c,c]^d$, then SIGM is $(\eps,\delta)$-DP with a noise level of order
\begin{equation*}
        \sigma^2 = \Theta\bigg(\frac{c^{2}\ln(1/\delta)}{n^{2}\gamma^{2}}+\frac{c^{2}d(\ln(d/\delta)+\varepsilon)\ln(d/\delta)}{n^{2}\varepsilon^{2}}\bigg),
\end{equation*}
where the cost per client is 
$\bo\left(\gamma d \log(2 + c/(\sigma\sqrt{n \gamma}))\right)$, 
and the error $\expec [\norm{Y - n^{-1} \sum_i x_i}^2_2]$ is at most $\frac{dc^2}{n\gamma} + d\sigma^2$.
\end{proposition}

\begin{remark}(Flattening and geometry)
To extend this result for data bounded in $\ell_2$ norm, one may rely on flattening techniques to convert the $\ell_2$ geometry into an $\ell_{\infty}$ geometry and obtain tighter utility analysis. Using Kashin's representation as in \citet{chen2023privacy} or some Walsh-Hadamard/Fourier transform as in \citet{kairouz2021distributed}, the SIGM mechanism enjoys the optimal utility bound of order $\bo(d/(n^2 \varepsilon^2))$. Observe that these flattening schemes require $\bo(d^2)$ or $\bo(d \log d)$ depending on the operation.
\end{remark}

\begin{table}[t]
\begin{centering}
{
\resizebox{\columnwidth}{!}{
\begin{tabular}{|l|c|c|c|c|}
\hline 
Quantized Aggregation Scheme & $\!\!\!\!\!\!\!\!$ $\begin{array}{c}
\text{Homo-}\\
\text{morphic}
\end{array}$ $\!\!\!\!\!\!\!\!$ & $\!\!\!\!\!\!\!$$\begin{array}{c}
\text{Gaussian}\\
\text{noise}
\end{array}$$\!\!\!\!\!\!$ & $\!\!\!\!\!\!$
$\begin{array}{c}
\text{Rényi}\\
\text{DP}
\end{array}$
$\!\!\!\!\!\!$ & $\!\!\!\!\!\!$$\begin{array}{c}
\text{Fixed}\\
\text{length}
\end{array}$$\!\!\!\!\!\!$\tabularnewline
\hline 
$\!\!$Individual - Direct (Def.\ref{def:direct})$\!\!$ & \xcross & \greenV & \greenV & \xcross\tabularnewline
\hline 
$\!\!$Individual - Shifted (Def.\ref{def:shifted})$\!\!$ & \xcross & \greenV & \greenV & \greenV\tabularnewline
\hline 
$\!\!$Irwin-Hall (Section \ref{subsec:irwinhall})$\!\!$ & \greenV & \xcross & \xcross & \greenV\tabularnewline
\hline 
$\!\!$Aggregate Gaussian (Def.\ref{def:decomp_mec})$\!\!$ & \greenV & \greenV & \greenV & \xcross\tabularnewline
\hline 
$\!\!$Subsampled ind. Gaussian (Sec. \ref{sec:5_dP_with_comp}) $\!\!$ & \xcross & \greenV & \greenV & \greenV\tabularnewline
\hline 
\end{tabular}}}
\par\end{centering}
\caption{Comparison of aggregate \QD{} mechanisms -- whether they are homomorphic, can produce a Gaussian noise, achieve Rényi DP, and have a fixed number of communication bits used. \vspace{-0.1cm}}
\end{table}

\textbf{Numerical comparison.} 
We empirically evaluate our distributed mean estimation scheme SIGM with the method CSGM in a similar framework as in \citet{chen2023privacy}. The parameter configuration is: number of clients $n \in \{1000;2000\}$, dimension $d\in \{100;500\}$, probability of information accidentally being leaked $\delta=10^{-5}$,  privacy budget $\varepsilon \in [0.5;4]$ and subsampling parameter $\gamma \in \{0.3;0.5;1.0\}$. 
Figure \ref{fig:trusted} displays the mean squared errors (MSE) obtained over $100$ independent runs. The number of bits used by CSGM is kept equal to the number of bits used by SIGM. For the same $\varepsilon$, $\gamma$ and the number of bits used, SIGM allows a smaller MSE compared to the CSGM. 

 \begin{figure}[t]
  \centering
  \subfigure[$n=1000, \ d=100$]{
  \includegraphics[scale=0.33]{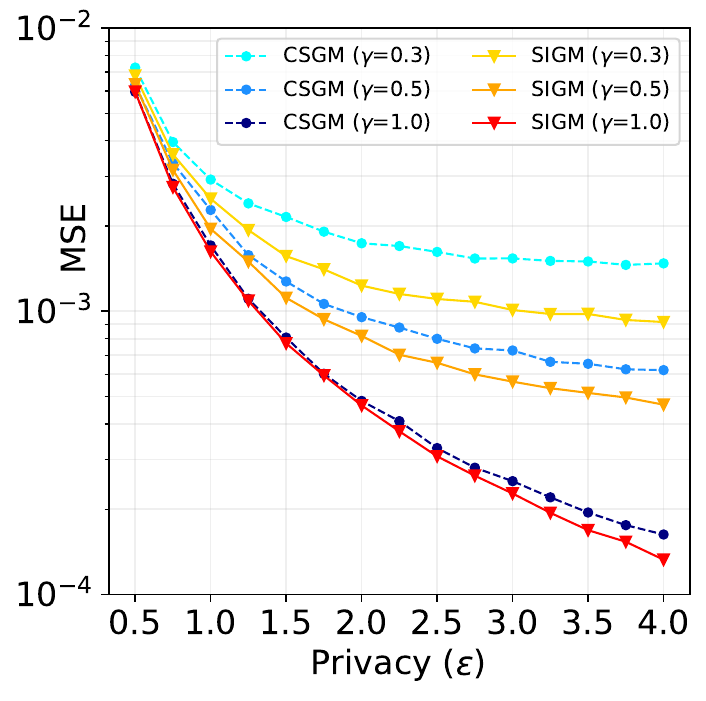}\label{fig:trusted1}}
  \subfigure[$n=2000, \ d=500$]{
  \includegraphics[scale=0.33]{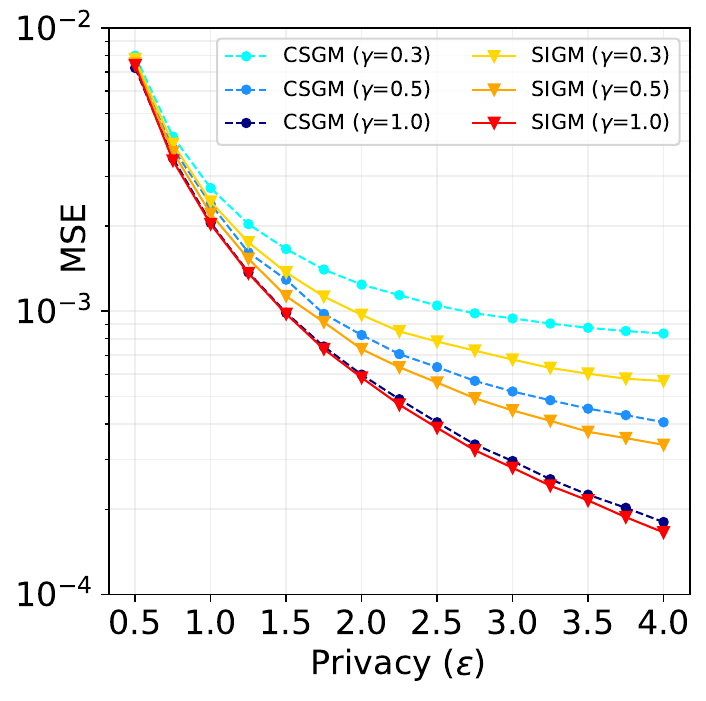}\label{fig:trusted2}}
    \vspace{-0.4cm}
    \caption{ Mean Squared Error for CSGM and SIGM.}
    \label{fig:trusted}
    \vspace{-0.3cm}
\end{figure}

\subsection{Less-trusted with SecAgg}

The Distributed Discrete Gaussian (DDG) mechanism \citep{kairouz2021distributed} can leverage SecAgg to achieve differential privacy guarantees against the server, which is a stronger setting than that of the \textit{less-trusted server}. However, in practical implementation, it often requires a much higher number of bits per coordinate.

We adapt the experiments from \citet{kairouz2021distributed} which show the utility of the DDG mechanism with different number of bits against the standard Gaussian mechanism. On the other hand, using Elias gamma coding, we calculate the number of bits needed for the \textit{aggregate Gaussian mechanism} and the \textit{individual mechanism} using \textit{shifted layered quantizer} to match a Gaussian mechanism. The results\footnote{The MSE curves for the DDG experiments are obtained with the original code of \citet{kairouz2021distributed}.} of Figure \ref{fig:ddg} are obtained over $30$ runs with $n=500$ and $d=75$ and highlight the great performance of aggregate Gaussian. While DDG offers stronger privacy guarantees, \textit{i.e.}, DP against the server, it comes at a heavy cost in terms of bits in comparison to aggregate Gaussian, which is also compatible with SecAgg. 

\begin{figure}[h!]
    \centering
    \vspace{-0.2cm}
    \includegraphics[scale=0.345]{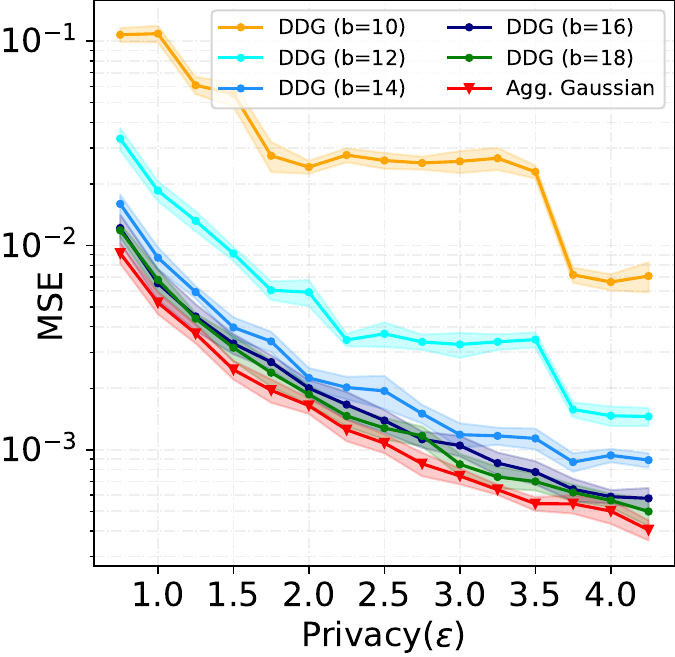}
    \includegraphics[scale=0.34]{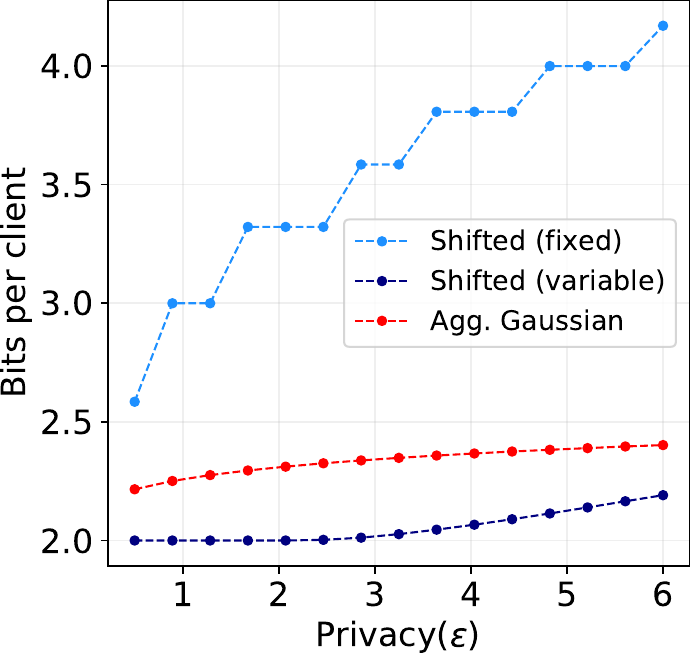}
    \vspace{-0.3cm}
    \caption{MSE (left) and bits per client (right) against $\varepsilon$. The DDG mechanism can require up to $b=18$ bits to match the privacy-utility tradeoff of aggregate Gaussian, where the latter only requires $\le 2.5$ bits on average. We also plot the bits per client for shifted layered quantizer (using a fixed or variable-length code) on the right figure for comparison (we remark that shifted layered quantizer is incompatible with SecAgg).}
    \label{fig:ddg}
    \vspace{-0.3cm}
\end{figure}


\newpage

\section*{Acknowledgements}

The work of M. Hegazy, R. Leluc  and A. Dieuleveut was partially supported by Hi!Paris FLAG project. The work of Cheuk Ting Li was partially supported by an ECS grant from the Research Grants Council of the Hong Kong Special Administrative Region, China [Project No.: CUHK 24205621].

\bibliographystyle{plainnat}
\bibliography{main.bbl}

\onecolumn
\newpage
\begin{center}
\Large
Appendix: \\
Compression with Exact Error Distribution for Federated Learning
\end{center}

Appendix \ref{app:additional_agg} is dedicated to additional results on aggregated mechanisms with a focus on the relative mixture entropy and the details of the algorithms used in Section \ref{sec:agg_mec}. Appendix \ref{app:proofs} gathers all the proofs of the theoretical results. Appendix \ref{app:supp_numerical} presents additional details and results for the numerical experiments. Appendix \ref{app:compression_smoothing} presents how the proposed quantizers with exact error distribution can be applied to obtain optimal algorithms for non-smooth distributed optimization problems.

\appendix
\startcontents[sections]
\printcontents[sections]{l}{1}{\setcounter{tocdepth}{2}}

\newpage
\section{Additional results on aggregated mechanisms} \label{app:additional_agg}

\subsection{Properties of relative mixture entropy}\label{app:relmixent}

\begin{proposition}
\label{prop:mixture_prop}The relative mixture entropy satisfies:
\begin{enumerate}
\item (Shifting and scaling) Let $X,Y$ be random variables, and denote
the distribution of $X$ as $P_{X}$. For constants $a,c\neq0$ and
$b,d\in\mathbb{R}$,
\[
h_{\mathrm{M}}(P_{cY+d}\Vert P_{aX+b})=h_{\mathrm{M}}(P_{Y}\Vert P_{X})+\log\frac{|c|}{|a|}.
\]
\item (Concavity in $Q$) For a fixed $P$, $h_{\mathrm{M}}(Q\Vert P)$
is concave in $Q$.
\item (Chain rule) For distributions $P,Q,R$, 
\[
h_{\mathrm{M}}(R\Vert P)\ge h_{\mathrm{M}}(R\Vert Q)+h_{\mathrm{M}}(Q\Vert P).
\]
\item (Bound via differential entropy) If $P,Q$ are continuous distributions,
\[
h_{\mathrm{M}}(Q\Vert P)\le h(Q)-h(P),
\]
where $h(P)$ is the differential entropy of $P$. 
\end{enumerate}
\end{proposition}

\begin{proof}
Property (1) is straightforward. To show concavity (2), consider $\lambda\in[0,1]$ and two distributions
$Q_{1},Q_{2}$, and consider $(A_{1},B_{1})$ independent of $(A_{2},B_{2})$
independent of $Z\sim P$ such that $A_{1}Z+B_{1}\sim Q_{1}$, $A_{2}Z+B_{2}\sim Q_{2}$.
Take $(A,B)=(A_{1},B_{1})$ with probability $\lambda$, and $(A,B)=(A_{2},B_{2})$
with probability $1-\lambda$. We have 
\begin{align*}
    AZ+B\sim\lambda Q_{1}+(1-\lambda)Q_{2}, \qquad \mathbb{E}[\log|A|]=\lambda\mathbb{E}[\log|A_{1}|]+(1-\lambda)\mathbb{E}[\log|A_{2}|].
\end{align*}
Taking supremum gives us the desired result. 
To show the chain rule (3), consider $(A_{1},B_{1})$ independent of $(A_{2},B_{2})$ independent
of $Z\sim P$ such that $Y=A_{1}Z+B_{1}\sim Q$, and $A_{2}Y+B_{2}\sim R$.
We have $A_{1}A_{2}Z+A_{2}B_{1}+B_{2}\sim R$, and $\mathbb{E}[\log|A_{1}A_{2}|]=\mathbb{E}[\log|A_{1}|]+\mathbb{E}[\log|A_{2}|]$.
Taking supremum gives us the desired result. To show the bound via differential entropy (4), if $(A,B)$ is independent of $Z\sim P$, then $AZ+B\sim Q$ and
\begin{align*}
h(Q)=h(AZ+B)\ge h(AZ+B\,|\,A,B)=h(Z)+\mathbb{E}[\log|A|].   
\end{align*}
\end{proof}

\subsection{Algorithm \textsc{DecomposeUnif}} \label{subsec:decomp_uni}

As an intermediate step, we study how we can simulate $\mathcal{U}(-1/2,1/2)$ using a noise with p.d.f. function $f$ that is unimodal, symmetric around $0$, and supported over $[-1/2,1/2]$ (e.g. a scaled version of Irwin-Hall). The idea is to recursively express uniform distributions as mixtures of shifted and scaled versions of $f$ and other uniform distributions, and repeat ad infinitum. 

We now describe the operation of decomposing $\mathcal{U}(-1/2,1/2)$ into a mixture of shifted and scaled versions of $f(x)$. For the sake of simplicity, assume $f(x)$ is unimodal, symmetric around $0$, and supported over $[-1/2,1/2]$ (e.g. a scaled version of Irwin-Hall). First, we can express the uniform distribution $\mathcal{U}(-1/2,1/2)$ as a mixture of $f(x)$ and other uniform distributions as follows. Generate $U\sim\mathcal{U}(-1/2,1/2)$ independent of $V\sim\mathcal{U}(0,1)$. If $V\le f(U)/f(0)$, then generate $X\sim f$. If $V>f(U)/f(0)$ and $U>0$, then take $S=f^{-1}(Vf(0))$ (where $f^{-1}(y):=\inf\{x\ge0:\,f(x)\le y\}$), and generate $X\sim\mathcal{U}(S,1/2)$. If $V>f(U)/f(0)$ and $U<0$, then take $S=f^{-1}(Vf(0))$, $X\sim\mathcal{U}(-1/2,-S)$. Then we can show that $X\sim\mathcal{U}(-1/2,1/2)$. We can then recursively express these uniform distributions in the mixture as mixtures of shifted and scaled versions of $f(x)$ and other uniform distributions, and repeat ad infinitum.

Algorithm \textsc{DecomposeUnif} takes $f$ and a random number generator $\mathfrak{P}$ as input (where we can invoke $\mathfrak{P}()$ to obtain a $\mathcal{U}(0,1)$ random number), and outputs two random variables $A,B$ such that if we generate $X\sim f$, then $AX+B\sim\mathcal{U}(-1/2,1/2)$.

\begin{algorithm}
\textbf{$\;\;\;\;$Input:} pdf $f$, random number generator $\mathfrak{P}$

\textbf{$\;\;\;\;$Output:} scale $a\in(0,\infty)$ and shift $b\in\mathbb{R}$

\smallskip{}

\begin{algorithmic}[1]

\State{$a\leftarrow1$, $b\leftarrow0$}

\While{True}

\State{$u\leftarrow\mathfrak{P}()-1/2$}

\State{$v\leftarrow\mathfrak{P}()$}

\If{ $v\le f(u)/f(0)$}

\State{\Return$(a,b)$}

\Else 

\State{$s\leftarrow f^{-1}(vf(0))$}

\State{$b\leftarrow b+a\cdot\mathrm{sgn}(u)\left(\frac{s+1/2}{2}\right)$}

\State{$a\leftarrow a(1/2-s)$}

\EndIf

\EndWhile

\end{algorithmic}

\caption{\label{alg:decomp_unif}$\textsc{DecomposeUnif}(f,\mathfrak{P})$}
\end{algorithm}

\subsection{Theoretical analysis of Algorithm $\textsc{DecomposeUnif}$}

We now analyze the value of $\mathbb{E}[\log A]$ given by Algorithm \textsc{DecomposeUnif}, which will give a bound on $h_{\mathrm{M}}(\mathcal{U}(-1/2,1/2)\Vert P)$.
We invoke the following result which follows from \citep[Theorem 1]{hegazy2022randomized}.
\begin{theorem}
For any probability density function $f$, we have $$h_{\mathrm{M}}(f\,\Vert\,\mathcal{U}(-1/2,1/2))\le h_{\mathrm{L}}(f)$$ and equality holds if $f$ is a unimodal distribution\footnote{A probability density function $f$ is unimodal if there exists $x_{0}\in\mathbb{R}$
such that $f(x)$ is nondecreasing for $x\le x_{0}$, and nonincreasing
for $x\ge x_{0}$.} with a finite mean, where $h_{\mathrm{L}}(f):=\int_{0}^{\infty}\lambda(L_{\tau}^{+}(f))\log\lambda(L_{\tau}^{+}(f))\mathrm{d}\tau$ is called the \emph{layered entropy} of $f$, where $L_{\tau}^{+}(f):=\{x\in\mathbb{R}:\,f(x)\ge\tau\}$ is the superlevel set, and $\lambda(L_{\tau}^{+}(f))$ is the Lebesgue
measure.
\end{theorem}

If $(X,Y)\sim\mathcal{U}\{(x,y):\,0\le y\le f(x)\}$, then $h_{\mathrm{L}}(f)=-h(Y)$
\citep{hegazy2022randomized}. Therefore, $h_{\mathrm{L}}(f)\ge-\log\sup_{x\in\mathbb{R}}f(x)$
is lower bounded by the differential min-entropy of $f$. We have
the following bound.

\begin{corollary}
\label{cor:unif_to_f} For any unimodal p.d.f. $f$ with a finite mean: $h_{\mathrm{M}}(f\,\Vert\,\mathcal{U}(-1/2,1/2)\ge-\log\sup_{x\in\mathbb{R}}f(x)$.
\end{corollary}

We can now bound $h_{\mathrm{M}}(\mathcal{U}(-1/2,1/2)\Vert P)$. The following lemma will be useful in the proof of Theorem~\ref{thm:hM_symm}.

\begin{lemma}
\label{lem:f_to_unif}For any unimodal probability density function
$f$ with $f^{*}:=\sup_{x}f(x)$, with a bounded support $L:=\sup\{x:f(x)>0\}-\inf\{x:f(x)>0\}$,
we have $h_{\mathrm{M}}(\mathcal{U}(-1/2;1/2)\,\Vert\,f)\ge-Lf^{*}-\log(eL/2).$
\end{lemma}
\begin{proof}

Write $\mathrm{U}:=\mathrm{Unif}(-1/2,1/2)$. By the shifting and
scaling property in Proposition \ref{prop:mixture_prop}, we can assume
$L=1$ and $f$ is supported over $[-1/2,1/2]$ without loss of generality.
Assume the maximum of $f$ is attained at $x_{0}$. We can express
$\mathrm{Unif}(-1/2,1/2)$ as a mixture \vspace{-0.2cm}
\[
\frac{1}{f^{*}}f(x)+\nu_{1}g_{1}(x)+\nu_{2}g_{2}(x), \quad \text{ where } \nu_{1}:=\int_{-1/2}^{x_{0}}\left(1-\frac{f(x)}{f^{*}}\right)\mathrm{d}x, \;\;\; \nu_{2}:=\int_{x_{0}}^{1/2}\left(1-\frac{f(x)}{f^{*}}\right)\mathrm{d}x,
\]
and 
\[
g_{1}(x):=\frac{1-f(x)/f^{*}}{\nu_{1}}\;\;\text{for}\;-\frac{1}{2}\le x\le x_{0},\;\;\; g_{2}(x):=\frac{1-f(x)/f^{*}}{\nu_{2}}\;\;\text{for}\;x_{0}\le x\le\frac{1}{2}.
\]
Note that both $g_{1}(x)$ and $g_{2}(x)$ are unimodal, and $\nu_{1}+\nu_{2}=1-1/f^{*}$.
We have 
\begin{align*}
h_{\mathrm{M}}(\mathrm{U}\Vert f) & \stackrel{(a)}{\ge}\nu_{1}h_{\mathrm{M}}(g_{1}\Vert f)+\nu_{2}h_{\mathrm{M}}(g_{2}\Vert f)\\
 & \stackrel{(b)}{\ge}\nu_{1}\left(h_{\mathrm{M}}(g_{1}\Vert\mathrm{U})+h_{\mathrm{M}}(\mathrm{U}\Vert f)\right)+\nu_{2}\left(h_{\mathrm{M}}(g_{2}\Vert\mathrm{U})+h_{\mathrm{M}}(\mathrm{U}\Vert f)\right)\\
 & =\nu_{1}h_{\mathrm{M}}(g_{1}\Vert\mathrm{U})+\nu_{2}h_{\mathrm{M}}(g_{2}\Vert\mathrm{U})+\left(1-\frac{1}{f^{*}}\right)h_{\mathrm{M}}(\mathrm{U}\Vert f).
\end{align*}
where (a) is by concavity, and (b) is by the chain rule (Proposition
\ref{prop:mixture_prop}). Hence,
\begin{align*}
 h_{\mathrm{M}}(\mathrm{U}\Vert f) &\ge f^{*}\nu_{1}h_{\mathrm{M}}(g_{1}\Vert\mathrm{U})+f^{*}\nu_{2}h_{\mathrm{M}}(g_{2}\Vert\mathrm{U})\\
 & \stackrel{(c)}{\ge}f^{*}\nu_{1}\log\nu_{1}+f^{*}\nu_{2}\log\nu_{2}\\
 & \ge f^{*}2\left(\frac{1-1/f^{*}}{2}\right)\log\frac{1-1/f^{*}}{2}\\
 & =(f^{*}-1)\log\frac{f^{*}-1}{2f^{*}}\\
 & =-(f^{*}-1)-(f^{*}-1)\log\left(1+\frac{1}{f^{*}-1}\right)\\
 & \ge-(f^{*}-1)-(f^{*}-1)\frac{\log e}{f^{*}-1}\\
 & =-f^{*}-\log\frac{e}{2},
\end{align*}
where (c) is by Corollary \ref{cor:unif_to_f}.
\end{proof}

\subsection{Algorithm \textsc{Decompose}} \label{subsec:decomp}

The goal is to simulate an additive noise channel with a noise p.d.f. $g$ that is unimodal, symmetric around $0$ and differentiable knowing that we can simulate an additive noise channel with a noise p.d.f. $f$ with same properties. The naive and quite inefficient way is to decompose $g$ into a mixture of uniform distributions, and run the $\textsc{DecomposeUnif}$ algorithm on those uniform distributions. Instead, we first decompose $g$ into the mixture $$g(x)=\lambda f(x)+(1-\lambda)\psi(x)$$
where $\lambda$ is as large as possible such that $\psi$ is still unimodal. A practical choice is $\lambda=\inf_{x>0} \mathrm{d}g(x)/\mathrm{d}f(x)$ if $n\ge3$, and $\lambda=0$ if $n\le2$. We then decompose $\psi$ into a mixture of uniform distributions, and run $\textsc{DecomposeUnif}$. Algorithm $\textsc{Decompose}$ below computes the decomposition of $g$ into a mixture of shifted and scaled versions of $f$. It takes $f,g$ and a random number generator $\mathfrak{P}$ as input, and outputs two random variables $A,B$ such that if we generate $X\sim f$, then $AX+B\sim g$. 

\begin{algorithm}[h]
\textbf{$\;\;\;\;$Input:} pdfs $f,g$, random number generator $\mathfrak{P}$

\textbf{$\;\;\;\;$Output:} scale $a\in(0,\infty)$ and shift $b\in\mathbb{R}$

\smallskip{}

\begin{algorithmic}[1]

\State{$\lambda\leftarrow\inf_{x>0}\frac{\mathrm{d}g(x)}{\mathrm{d}f(x)}$}

\State{Sample $x\sim g$ using $\mathfrak{P}$}

\State{$v\leftarrow g(x)\cdot\mathfrak{P}()$}

\If{ $v>g(x)-\lambda f(x)$}

\State{\Return$(1,0)$}

\Else

\State{$s\leftarrow\sup\{x'\ge0:\,v\le g(x')-\lambda f(x')\}$}

\State{$L\leftarrow2\sup\{x:f(x)>0\}$}

\State{$\tilde{f}\leftarrow(x\mapsto f(x/L)/L)$ (support of $\tilde{f}$ is $[-1/2,1/2]$)}

\State{$(a,b)\leftarrow\textsc{DecomposeUnif}(\tilde{f},\mathfrak{P})$}

\State{\Return$(2as/L,\,2bs)$}

\EndIf

\end{algorithmic}

\caption{\label{alg:decomp}$\textsc{Decompose}(f,g,\mathfrak{P})$}
\end{algorithm}

\subsection{Algorithms for Aggregate Gaussian Mechanism}\label{subsec:agg_gauss_algo}

We describe the encoding and decoding algorithms for the aggregate Gaussian mechanism. We assume the server and all the clients share a common random seed, which they use to initialize their pseudorandom number generators (PRNG) $\mathfrak{P}$. Since each PRNG is initialized with the same seed, they are guaranteed to produce the same common randomness $(S_{i})_{i},A,B$ at the clients and server.

\begin{minipage}{0.49\textwidth}
\centering
\begin{algorithm}[H]
\begin{algorithmic}[1]
\State{\textbf{Input:} data $x\in\mathbb{R}$, number of clients $n$, client id $i$, standard deviation $\sigma$, RNG $\mathfrak{P}$}
\State{\textbf{Output:} description $m\in\mathbb{Z}$}
\State{$(a,b)\leftarrow\textsc{Decompose}(\mathrm{IH}(n,0,1),\mathcal{N}(0,1),\mathfrak{P})$}
\For{$j=1,\ldots,n$}
\State{$s_{j}\leftarrow\mathfrak{P}()-1/2$}
\EndFor
\State{$w\leftarrow2\sigma\sqrt{3n}$}
\State{\Return$\left\lceil \frac{x}{a w}+s_{i}\right\rfloor $}
\end{algorithmic}
\caption{\label{alg:agg_enc}$\textsc{Encode}(x,n,i,\sigma,\mathfrak{P})$}
\end{algorithm}
\end{minipage}
\hfill
\begin{minipage}{0.49\textwidth}
\centering
\begin{algorithm}[H]
\begin{algorithmic}[1]
\State{\textbf{Input:} sum of descriptions $\Sigma_{m}=\sum_{i=1}^{n}m_{i}\in\mathbb{Z}$, number of clients $n$, standard deviation $\sigma$, RNG $\mathfrak{P}$}
\State{\textbf{Output:} estimated mean $y\in\mathbb{R}$}
\State{$(a,b)\leftarrow\textsc{Decompose}(\mathrm{IH}(n,0,1),\mathcal{N}(0,1),\mathfrak{P})$}

\For{$j=1,\ldots,n$}

\State{$s_{j}\leftarrow\mathfrak{P}()-1/2$}

\EndFor

\State{$w\leftarrow2\sigma\sqrt{3n}$}

\State{\Return$\frac{a w}{n}\left(\Sigma_{m}-\sum_{i=1}^{n}s_{i}\right)+b\sigma$}

\end{algorithmic}
\caption{\label{alg:agg_dec}$\textsc{Decode}(\Sigma_{m},n,\sigma,\mathfrak{P})$}
\end{algorithm}
\end{minipage}

\subsection{Algorithm SIGM for Differential Privacy Applications}\label{subsec:sigm_alg}

Condition on any $B_1,\ldots,B_n$. Definition \ref{def:shifted} ensures that $\mathscr{D}(M_i,S_i)-x_i \sqrt{\tilde{n}} \sim \mathcal{N}(0,(\sigma \gamma n)^2)$. Therefore
\begin{align*}
Y - (\gamma n)^{-1}\sum_{i: B_i=1} x_i &= (\gamma n \sqrt{\tilde{n}})^{-1}\sum_{i:\,B_i=1}\Big( \mathscr{D}(M_i,S_i) - x_i \sqrt{\tilde{n}}\Big) \\
& \sim \mathcal{N}\Big(0,(\sigma \gamma n)^2 (\gamma n \sqrt{\tilde{n}})^{-2} \tilde{n} \Big) \; =\; \mathcal{N}(0,\sigma^2 ).
\end{align*}

To generalize to $d$ dimensional data $x_1,\ldots,x_n \in \mathbb{R}^d$, we can simply apply SIGM on each coordinate individually. Note that each coordinate $j\in [d]$ has $B_1(j),\ldots,B_n(j)$ sampled independently, indicating whether client $i=1,\ldots,n$ should send the coordinate $x_i(j)$, similar to the coordinate subsampling strategy in \citep{chen2023privacy}. The algorithm is given below.

\begin{algorithm}[h!]
    \label{alg:quant_subsampling}
  \caption{Subsampled Individual Gaussian Mechanism (SIGM)}
  \begin{algorithmic}
    \State \hspace{-0.2cm}  \textbf{Parameters:} Noise variance $\sigma^2>0$, Subsampling Bernoulli parameter $\gamma \in [0,1]$ \\
    \begin{center} \textbf{(Shared Randomness)} \end{center}
        \State \hspace{-0.2cm} $B_1,\ldots,B_n \in \ens{0,1}^d$ with $B_i(j)\sim \mathcal{B}(\gamma)$
        \State \hspace{-0.2cm} $S_1,\ldots, S_n:$ $S_i(1,j)\sim \Uni\open{0,1},S_i(2,j)\sim W_{\mathcal{N}(0,(\sigma \gamma n)^2)}$ 
        \State \hspace{-0.2cm} $\tilde{n}(j) \leftarrow \sum_{i=1}^n B_i(j)$ $\;\;\;$ for $j\in [d]$ \\
    \hrulefill
    \begin{center} \textbf{(Client side)} \end{center}
    \State \hspace{-0.2cm}  \textbf{input:} private $x_i \in \rset^d$ with $\|x_i\|_2 \leq c$ 
    \State \hspace{-0.2cm}  \textbf{Encode:} For $j \in [d]$
    \State \hspace{-0.2cm}  \hspace{0.5cm} If $B_i(j)=1$ then $M_i(j) \leftarrow   \ecal({x}_i(j) \sqrt{\tilde{n}(j)}, S_i(\cdot \;, j))$
    \State \hspace{-0.2cm}  \hspace{0.5cm} else $M_i(j) \leftarrow   0$
    \State \hspace{-0.2cm}  \textbf{Send:} $M_i$ to the server \\
    \hrulefill
    \begin{center} \textbf{(Server side)} \end{center}
    \State \hspace{-0.2cm}  \textbf{input:} $M_1,\ldots,M_n \in \mathbb{Z}^d$
    \State \hspace{-0.2cm}  \textbf{Decode:} For $j \in [d]$
    \State \hspace{-0.2cm}  \hspace{0.5cm} ${\bar \mu}(j) \leftarrow (\gamma n\sqrt{\tilde{n}})^{-1}\sum_{i:\,B_i(j)=1} \mathscr{D}(M_i(j),S_i(\cdot,j))$
    \State \hspace{-0.2cm}  \textbf{Return:} $\bar \mu$
  \end{algorithmic}
\end{algorithm}

\newpage
\section{Proofs of technical results}
\label{app:proofs}

\subsection{Shifted Layered Quantizer satisfies AINQ}
\label{subsec:equiv_quantizers}

Let us show that the shifted layered quantizer yields the same distribution of error scheme as the direct layered quantizer (Definition \ref{def:direct}). Let $\Delta=Y-X$, then $\Delta|V\sim \Uni\open{b^-_Z(\bar{Z}-V),b^+_Z(V)}$. Denote $\mathcal{M} = \argmax f_Z(x)$ and $m=\text{median}(\mathcal{M})$. It follows that
\begin{align*}
    f_\Delta(x) &= \int_\rset \frac{\one{\ens{x\in\open{b^-_Z(\bar{Z}-v),b^+_Z(v)}}}}{W_Z(v)}W_Z(v) dv \\
     &\stackrel{(a)}{=}  \int_\rset \one{\ens{x\in\open{b^-_Z(\bar{Z}-v),m}}}dv+ \int_\rset \one{\ens{x\in\open{m,b_Z^+(v)}}} dv\\ 
     &\stackrel{(b)}{=} \int_\rset \one{\ens{x\in \open{b_Z^-(v),m}}}dv + \int_\rset \one{\ens{x\in\open{m,b_Z^+(v)}}} dv\\
     &= \int_\rset\one{\ens{x\in \open{b_Z^-(v), b_Z^+(v)}}} dv\\
     &\stackrel{(c)}{=}  \int_\rset\one{\ens{x\in \mathcal{L}_v(f_Z)}} dv = f_Z(x)
\end{align*}
where $(a)$ is due to the fact that $m\in \mathcal{L}_\tau(f_Z)$ for any $\tau$, $(b)$ by the change of variable $v=\bar{Z}-v$, and $(c)$ by unimodality of $f_Z$.

\subsection{Proof of Proposition \ref{prop:gap2}}

The conditional entropy $ \ent(M_s|S_s)$ may be decomposed as follows
\begin{align*}
    \ent(M_s|S_s)-\log (t) &= -\log(t)+\int_0^1\int_0^{\bar{Z}} \ent(M_s|S_s=(u,\tau))f_W(\tau) \diff \tau \diff s \\
    &\leq -\log(t) + \int_0^1\int_0^{\bar{Z}} \log\open{\frac{t}{f_W(\tau)}+2} f_W(\tau) \diff \tau\\
    &= h(W_Z)+\int_0^{\bar{Z}} \log\open{1+\frac{2f_W(\tau)}{t}}f_W(\tau) \diff \tau\\
    &\leq h(W_Z)+\frac{2\log(e)}{t} \int_0^{\bar{Z}}f_W(\tau)^2 \diff \tau.
\end{align*}

The last integral on the right-hand side may be upper bounded by considering the random variables $(Z,S) \sim \mathcal{U}\{(z,s)|s\in\closed{a(z),b(z)}\}$ with $\gamma\in \argmax f_Z$ and 
\begin{align*}
    a(x) &= \one\{x < \gamma\}(\bar{Z} - f_Z(x)) \\
    b(x) &= \one\{x < \gamma\} f_Z(x) + \one\{x \geq \gamma\} \bar{Z}
\end{align*}

This coupling satisfies $S\sim W_Z$, $Z\sim f_Z$, $Z|S \sim \mathcal{U}[b_Z^-(\bar{Z}-S), b_Z^+(S)]$. Then if follows that 
\begin{align*}
    \int_0^{\bar{Z}} f_W(\tau)^2 \diff \tau 
    =& 4\int_0^{\bar{Z}} \expec\closed{\modu{Z-\frac{b_Z^+(\tau)+b_Z^-(\bar{Z}-\tau)}{2}} | S=\tau} f_W(\tau) \diff \tau.
\end{align*}
Since $Z$ is uniformly distributed over an interval conditional on $S=\tau$, we have that $\expec[|Z-m| | S=\tau]$ is minimized when $m$ is the midpoint of the interval so it holds that
\begin{align*}
\int_0^{\bar{Z}} f_W(\tau)^2 \diff \tau
 &\leq 4\int_0^{\bar{Z}} \expec\closed{\modu{Z-\text{median}(Z)} | S=\tau} f_W(\tau) \diff \tau \\
    &= 4 \expec\closed{\modu{Z-\text{median}(Z)}}
\end{align*}
and the last term may be upper bounded using that $\expec\closed{\modu{Z-\text{median}(Z)}} \leq \expec\closed{\modu{Z}}$ combined with Jensen inequality to finally obtain $\int_0^{\bar{Z}} f_W(\tau)^2 \diff \tau \leq 4 \sqrt{\var[Z]}$ which gives the final bound
\begin{equation*} 
 \ent(M|S) \leq \log(t)+\frac{8\log(e)}{t}\sqrt{\var\closed{Z}}+h(W_Z).
\end{equation*}
Observe that the optimality gap is
\begin{equation*}
    \frac{8\log(e)}{t}\sqrt{\var[Z]}+h(W_Z)-h(D_Z).
\end{equation*}
Thus it is sufficient to bound $h(W_Z)-h(D_Z)$. As $f_W$ and $f_D$ are translation invariant, then we assume without loss of generality that the mode of $f_Z$ is $0$. Using the symmetry of $f_Z$, we have $b^+_Z(x)=-b^-_Z(x)$. Then we recover
\begin{align*}
    h(D_Z)&=-\int_0^{\bar{Z}} 2b^+_Z(x)\log\open{2b^+_Z(x)} \diff x,\\
    h(W_Z)&=-\int_0^{\bar{Z}} 2b^+_Z(x)\log\open{b^+_Z(x)+b^+_Z(\bar{Z}-x)} \diff x.
\end{align*}
As $b^+_Z(x)\geq 0$, then
\begin{align*}
    h(W_Z) - h(D_Z) &\leq \int_0^{\bar{Z}} 2b^+_Z(x) \log\open{\frac{2b^+_Z(x)}{b^+_Z(x) + b^+_Z(\bar{Z}-x)}} \diff  x\\
    &= 2 \int_0^{\bar{Z}} D_Z(X) \diff x = 2
\end{align*}
Similarly as $f_W(X)\geq \eta_Z$, 
\begin{align*}
    h(W_Z) &\leq \int_0^{\bar{Z}} 2b^+_Z(x) \log\open{\frac{1}{\eta_Z}} \diff x\\
    &= -\log(\eta_Z)
\end{align*}

\subsection{Proof of Proposition \ref{prop:min_step2}}

For the Gaussian case, refer to \citet[Page~22]{wilson2000layered}. We prove the case for Laplace random variable with a similar strategy. Note that $f_W$ is translation invariant w.r.t to $f_Z$, so we consider zero mean Laplace. For $Z\sim \text{Laplace}\open{0,b}$, we have $f_W(x)=-b\ln(2bx)-b\ln(1-2bx)$. Solving $\mathrm{d}f_W(x^\star)/\mathrm{d}x=0$ gives $x^\star=1/4b$ so that $f_W\open{x^\star}=2b\ln2$ and plug the value $b=\sigma / \sqrt{2}$ to conclude.

\subsection{Proof of Proposition \ref{prop:decomp_ainq}}

Condition on $A=a,B=b$. Subtractive dithering gives
\[
\frac{aw}{n}\left(\left\lceil \frac{x_{i}}{aw}+S_{i}\right\rfloor -S_{i}\right)-\frac{x_{i}}{n}\sim\mathcal{U}\left(-\frac{aw}{2n},\,\frac{aw}{2n}\right).
\]
Hence,
\begin{align*}
Y-\frac{1}{n}\sum_{i=1}^{n}x_{i} & =\frac{aw}{n}\left(\sum_{i=1}^{n}\left\lceil \frac{x_{i}}{aw}+S_{i}\right\rfloor -\sum_{i=1}^{n}S_{i}\right)-\frac{1}{n}\sum_{i=1}^{n}x_{i}+b\\
 & \sim\mathrm{IH}\left(n,\,b,\,\frac{n}{12}\left(\frac{aw}{n}\right)^{2}\right)\\
 & =\mathrm{IH}\left(n,\,b,\,a^{2}\sigma^{2}\right)
\end{align*}
has the same distribution as $aZ+b$ where $Z\sim P=\mathrm{IH}(n,0,\sigma^{2})$.
If we randomize over $A,B$, since $(A,B)\sim\pi_{A,B}\in\Pi_{A,B}(P,Q)$,
we have $Y-\frac{1}{n}\sum_{i=1}^{n}x_{i}\sim Q$ by the definition of mixture set.

We then demonstrate that the mechanism is homomorphic. Since $\overline{\mathscr{D}}$ only depends on $m_1,\ldots,m_n$ through $\sum_i m_i$, it allows the server to decode using only $\sum_i m_i$ and the shared randomness. While $\overline{\mathscr{D}}$ is technically not in the form (\ref{eq:hom}) due to the ``$+b$'' term, that term it can be absorbed into $s_i$. Let $b'_i = b/n$, and treat $(s_i,b'_i)$ as the common randomness between client $i$ and the server. We then have $\overline{\mathscr{D}}((m_i)_i,(s_i,b'_i)_i,a)= \frac{a w}{n}\left(\sum_{i=1}^n m_i- \sum_{i=1}^n (s_i-b'_i)\right)$, which is in the form (\ref{eq:hom}) by taking $\mathscr{D}(m,(s,b'),a)=aw(m-s+b')$.

\medskip

\subsection{Proof of Proposition \ref{prop:dp_sigm}}

The SIGM algorithm returns the mean estimate $\bar \mu = (\bar \mu_1,\ldots, \bar \mu_d)$ where for $j \in [d]$ we have
\begin{align*}    
{\bar \mu}(j) =  (\gamma n\sqrt{\tilde{n}})^{-1}\sum_{i:\,B_i(j)=1} \mathscr{D}(M_i(j),S_i(\cdot,j)).
\end{align*}
Similarly to the proof in \ref{subsec:sigm_alg}, we have $\bar \mu_j - (n \gamma)^{-1} \sum_{i:B_{i}(j)=1} x_i(j) \sim \mathcal{N}(0,\sigma^2)$ and conclude by invoking the proof of \citep[Theorem 4.1]{chen2023privacy} to obtain the bound on $\sigma^2$.

For the communication cost per client, Proposition  \ref{prop:min_step2} gives 
\begin{align*}
    \modu{\supp M}\leq 2 + \frac{t}{2 \sigma \sqrt{n \gamma} \sqrt{\ln 4}}.
\end{align*}
Since the data $x_1,\ldots,x_n$ belong to $[-c;c]^d$ and the clients encode $x_i \sqrt{\tilde n}$, we have $t = 2 c \sqrt{\tilde n}$. Since $\expec[\tilde n] = \gamma n$ and only $\gamma d$ coordinates are selected on average, taking the log on both sides gives the desired result.


\subsection{Proof of Theorem \ref{thm:irwin_bound}\label{subsec:pf_irwin_bound}}

We have
\begin{align}
 \mathbb{E}\left[\left\lceil \log\left(\frac{t}{|A|d}+3\right)\right\rceil \right]\nonumber 
 & <\mathbb{E}\left[\log\left(1+\frac{3d|A|}{t}\right)\right]+\log\frac{t}{d}+\mathbb{E}[-\log|A|]+1\nonumber \\
 & \le-\mathbb{E}[\log|A|]+\frac{3d\log e}{t}\mathbb{E}\left[|A|\right]+\log\frac{t}{d}+1.\label{eq:pf_irwin_bound}
\end{align}
Recall that if $Z\sim P$ is independent of $A,B$, then $AZ+B\sim Q$.
We have
\begin{align*}
\mathbb{E}\left[|AZ+B|\right] & =\mathbb{E}\left[\mathbb{E}\left[|AZ+B|\,\Big|\,A,B\right]\right]\\
 & \stackrel{(a)}{\ge}\mathbb{E}\left[\mathbb{E}\left[|AZ|\,\Big|\,A,B\right]\right]\\
 & =\mathbb{E}\left[|A|\right]\mathbb{E}\left[|Z|\right],
\end{align*}
where (a) is because the median of $aZ$ is $0$ for any $a\in\mathbb{R}$,
so $\mathbb{E}[|aZ+b|]$ is minimized when $b=0$. Combining this
with (\ref{eq:pf_irwin_bound}) gives the desired result.

\subsection{Proof of Theorem \ref{thm:hM_symm}}\label{subsec:pf_hM_symm}

Write $\mathrm{U}:=\mathcal{U}(-1/2,1/2)$. Express $g(x)$ as a
mixture
$g(x)=\lambda f(x)+(1-\lambda)\psi(x)$.
We have
\begin{align*}
h_{\mathrm{M}}(g\Vert f) & \stackrel{(a)}{\ge}(1-\lambda)h_{\mathrm{M}}(\psi\Vert f)\\
 & \stackrel{(b)}{\ge}(1-\lambda)\left(h_{\mathrm{M}}(\psi\Vert\mathrm{U})+h_{\mathrm{M}}(\mathrm{U}\Vert f)\right)\\
 & \stackrel{(c)}{\ge}-(1-\lambda)\left(\log\frac{g(0)-\lambda f(0)}{1-\lambda}+Lf(0)+\log\frac{eL}{2}\right),
\end{align*}
where (a) is by concavity, (b) is by the chain rule (Proposition \ref{prop:mixture_prop}),
and (c) is by Corollary \ref{cor:unif_to_f} and Lemma \ref{lem:f_to_unif}.



\newpage
\section{Additional Numerical Experiments} \label{app:supp_numerical}
\textbf{Comparing AINQ mechanism with Quantization.}
Applying the shifted layered quantizer with a required compression error $\ncal(0,\sigma^2 \mathrm{I}_d)$ coordinate-wise to an input $x \in \rset^d$ that is unbounded does not necessarily lead to a bounded encoding $M$ (alternatively bounded number of bits to send $M$). However, a classical implementation of quantization starts by specifying the number of bits, then normalizing $x$ by $\norm{x}_p$ to scale it to the $\ell_p$ unit sphere and finally perform quantization on each coordinate using $b$ bits. For decompression, $\norm{x}_p$ is shared with the decoder to rescale the quantized vector. The communication cost of sharing $\norm{x}_p$ may be ignored as it is often negligible after normalizing by the dimension $d$. To have a fair comparison with classical quantization schemes, in DP experiments, shifted layered quantization is used to reproduce the desired error distribution. Then, the number of bits used by the shifted layered quantizer is measured on the one hand and the number of bits for quantization is adjusted on the other hand such that it is always equal or higher than the number of bits used for shifted layered quantizer.

For Langevin dynamics, the number of bits $b$ is specified first. Each client $i \in [n]$ scales its vector $x_i$ by $\norm{x_i}_\infty$. The variance $\sigma^2_b$ used for the compression is calculated using Proposition \ref{prop:min_step2} with $t=2$. Finally, decoding achieves a compression error of variance $\sigma_b^2 \norm{x}_\infty^2$. As $\sigma^2_b$ is known and the different norms $\ens{\norm{x_i}_\infty}_{i\in [n]}$ are shared, the server does some accounting to calculate the error variance of the distributed mean operation and adjusts the added error to reproduce the desired Markov chain. Refer to Algorithm \ref{algo:QLSD-star} for more details.

\textbf{Software/Hardware details.} The code to reproduce the curves in the Figures below is available upon request. Other experiments were reproduced through the original implementation of \citet{kairouz2021advances} and \citet{vono2022qlsd}. The experiments on differential privacy applications (see Appendix \ref{subset:num_dp}) were performed on a laptop Intel Core i7-10510U CPU 1.80GHz $\times$ 8. The experiments on stochastic Langevin dynamics (see Appendix \ref{app:langevin}) with toy gaussian data were performed on 4 cores of slurm-based cluster with 2.1 Ghz frequency. The computations for Langevin experiments required around 72 hours.

\subsection{Compression for free in Differential Privacy}
\label{subset:num_dp}

\textbf{Subsampling and Trusted Server.} The parameter configuration for the experiments in Section \ref{sec:5_dP_with_comp} are: number of clients $n \in \{1000;2000\}$, dimension $d\in \{100;500\}$. The numerical results of Figure \ref{fig:csgm_sigm_compare} below consider the setting $d=500$ and $n \in \{250;500;1000\}$. For a subsampling parameter $\gamma \in \{0.3;0.5;1.0\}$, we set $\delta=10^{-5}$ and privacy budget $\varepsilon \in [0.5;4]$. The data is generated in a similar spirit to \citet{chen2023privacy}. Denote by $X_{i}(j)$ the $j$-th coordinate of the $i$-th client data. Then, for $i \in [n]$ and for $j \in [d]$, $X_{i}(j) \sim (2 \mathcal{B}(p)-1)U/\sqrt{d}$ where $\mathcal{B}(p)$ is a Bernoulli variable with parameter $p=0.8$ and $U \sim \mathcal{U}(0,1)$ is a standard uniform variable. The multiplication by this uniform variable allows to work on continuous data $X$ whereas the experiments in \citep{chen2023privacy} consider discrete values. Figure \ref{fig:csgm_sigm_compare} below reports the results of additional experiments obtained over $100$ independent runs, where we plot the mean squared error curves against the privacy budget $\varepsilon$ in dimension $d=500$ with different number of clients $n \in \{250;500;1000\}$. 

 \begin{figure}[h!]
  \centering
  \subfigure[$n=250$]{
  \includegraphics[scale=0.45]{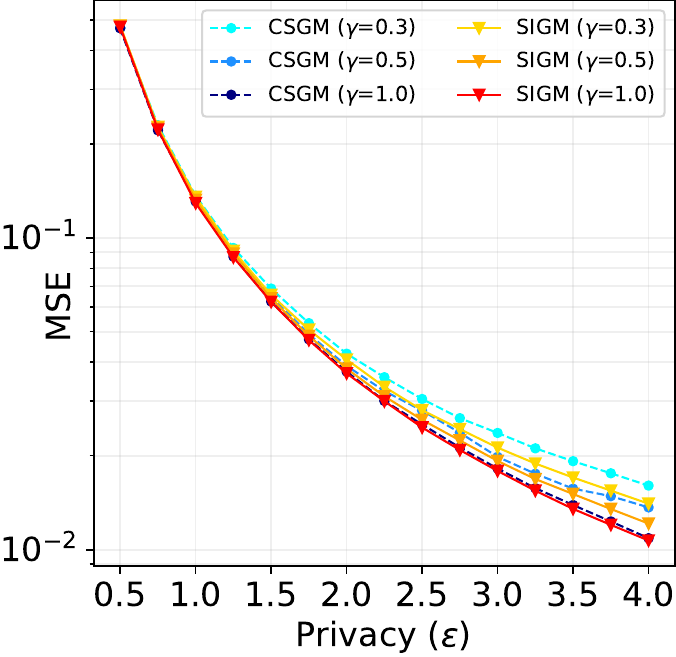}\label{fig:trusted_n100}}
  \subfigure[$n=500$]{
  \includegraphics[scale=0.45]{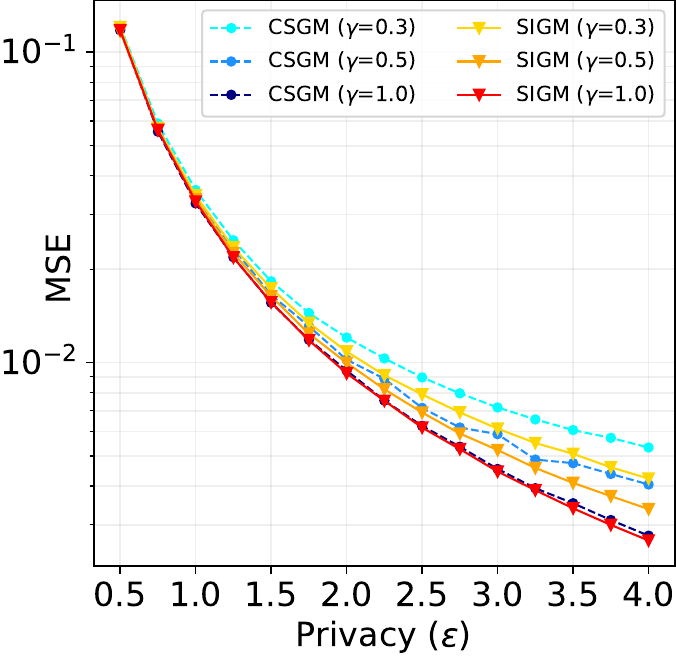}\label{fig:trusted_n200}}
  \subfigure[$n=1000$]{
  \includegraphics[scale=0.45]{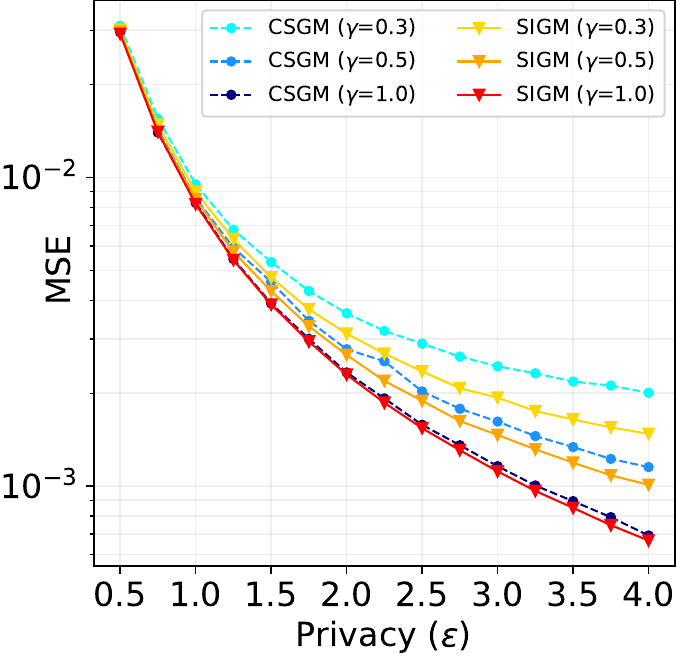}\label{fig:trusted_n500}}
    \caption{Mean Squared Error for CSGM and SIGM with $n \in \{250;500;1000\}$ and $d=500$.}
    \label{fig:csgm_sigm_compare}
\end{figure}

\pagebreak

\textbf{Less-Trusted Server.} The Distributed Discrete Gaussian (DDG) mechanism \citep{kairouz2021distributed} can leverage SecAgg to achieve DP guarantees against the server, which is a stronger setting than that of the \textit{less-trusted server}. However, in practical implementation, it often requires a much higher number of bits per coordinate. The DDG mechanism relies on the discrete Gaussian distribution \citep{canonne2020discrete} $\mathcal{N}_\zset(\mu, \sigma^2)$ with $\mu, \sigma \in \rset$ and $\sigma>0$. For any $x \in \zset$, $\prob_{\mathcal{N}_\zset}(X=x)$ is proportional to $\exp(-(x-\mu)^2/(2\sigma^2))$. Using this discrete Gaussian noise, the DDG mechanism works as follows. The input parameters are: scaling factor $\gamma$; clipping threshold $c>0$; bias $\beta\geq0$; modulus $m\in \nset$ and noise scale $\sigma^2$. The shared randomness is composed of a random unitary matrix $U \in \mathbb{C}^{d\times d}$. Each client $i \in [n]$ first scales and clips the data as $x_i^{'} =\gamma^{-1}\min \ens{1, c/\norm{x^i}_2}\cdot x^i$ then rotate $x_i^{''}$ = $U x_i^{'}$ and quantize $x_i^{''}$ to $M_i$ such that $\norm{M_i}_2\leq C(\gamma,\beta,c)$. The worker finally sends to the server $\tilde{M}_i = M_i + G_i \mod m$, where $G_i\sim \mathcal{N}_\zset (0,\sigma^2)$ using SecAgg. The server receives the results of SecAgg $\tilde{M} = \sum_i^n \tilde{M}_i \mod  m $ and outputs $(\gamma/n)U^*M$ (rescale and inverse the rotation). For more details, refer to Algorithms 1 and 2 in \citet{kairouz2021distributed}. Note that some recent work of \citet{chen2022fundamental} aims at improving communication efficiency of DDG by applying projections to reduce the dimension of the data transmitted. However, it still uses DDG as a subroutine so we restrict our comparison to DDG.

To compare with the utility-communication trade-off of DDG, we adapt the experiments from \citet{kairouz2021distributed}. On the one hand, we reproduce the figures comparing the utility of DDG with different number of bits against the standard Gaussian mechanism. On the other hand, using Elias gamma coding, we measure over $50$ runs the average number of bits needed for the {aggregate Gaussian mechanism} and the {shifted layered quantizer} to match the Gaussian mechanism. Note that only the {aggregate Gaussian mechanism} is compatible with SecAgg. Figures \ref{fig:ddg_compare} and \ref{fig:shifted_compare} below presents the mean squared error curves for different number of clients $n \in \{100;500;1000\}$ and highlights that AINQ mechanisms requires much lower number of bits to realize the Gaussian mechanism. The graphs in Figure \ref{fig:shifted_compare} compare the number of bits per client for the shifted layered quantizer (fixed or variable length) and the aggregate Gaussian mechanism, for a privacy budget $\varepsilon \in [1,10]$. 
\vspace{-0.2cm}
 \begin{figure}[h!]
  \centering
  \subfigure[$n=100$]{
  \includegraphics[scale=0.43]{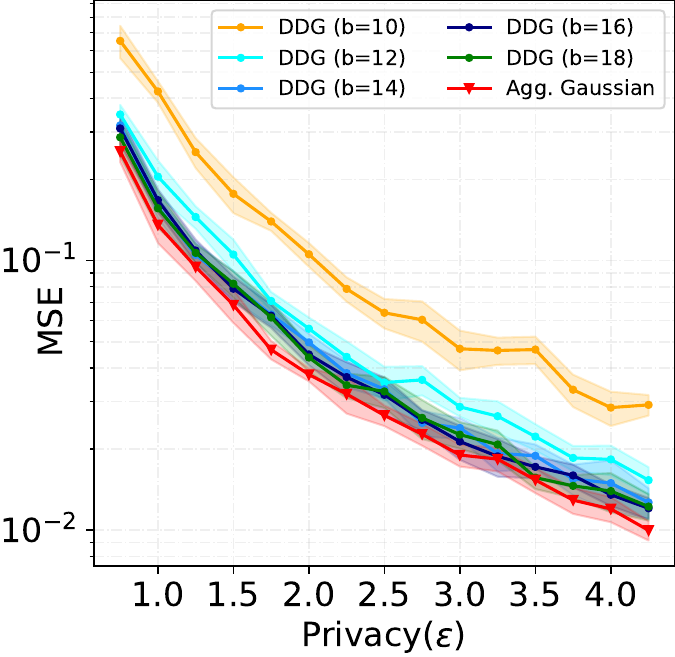}\label{fig:ddg_n100}}
  \subfigure[$n=500$]{
  \includegraphics[scale=0.43]{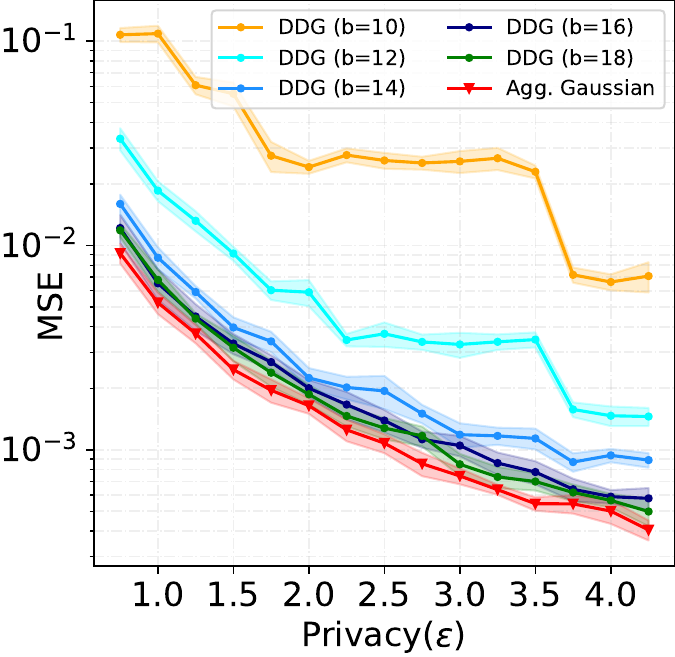}\label{fig:ddg_n500}}
  \subfigure[$n=1000$]{
  \includegraphics[scale=0.43]{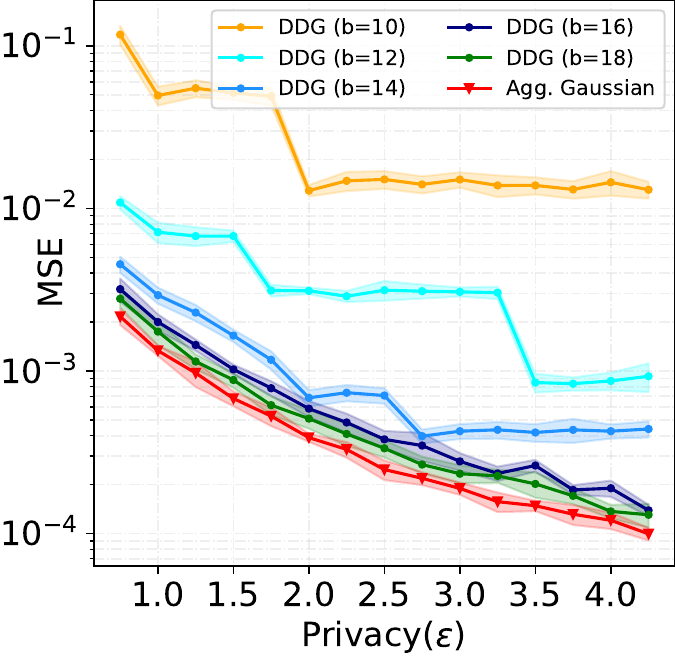}\label{fig:ddg_n1000}}
  \vspace{-0.33cm}  \caption{Mean Squared Error for the DDG mechanism and Aggregate Gaussian mechanism for $n \in \{100;500;1000\}$, $\delta=10^{-5}$ and $d=75$. Data samples drawn from the $\ell_2$ sphere with radius $c=10$.}
    \label{fig:ddg_compare}
\end{figure}

 \begin{figure}[h!]
  \centering
  \subfigure[Aggregate Gaussian]{
  \includegraphics[scale=0.38]{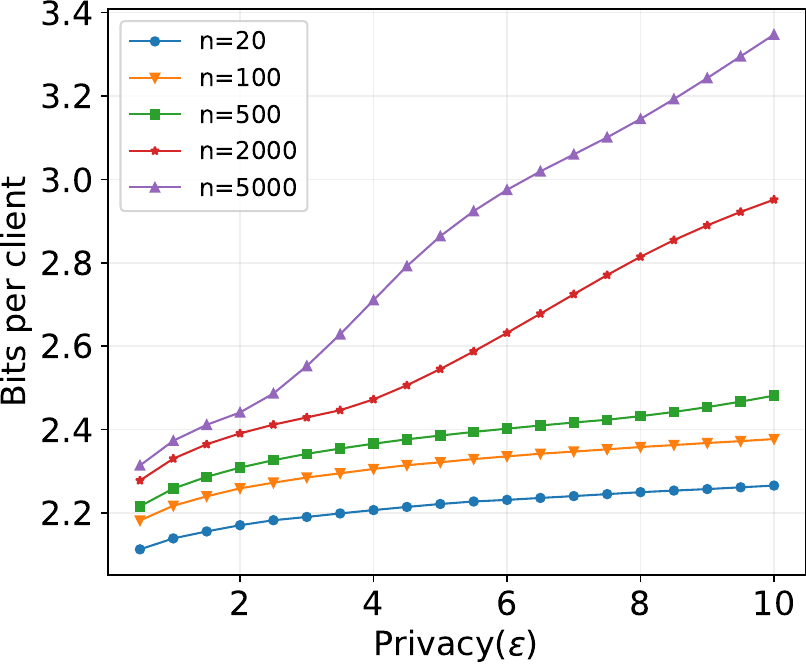}\label{fig:avg_agg}}
  \subfigure[Shifted Layered - Upper Bound]{
  \includegraphics[scale=0.38]{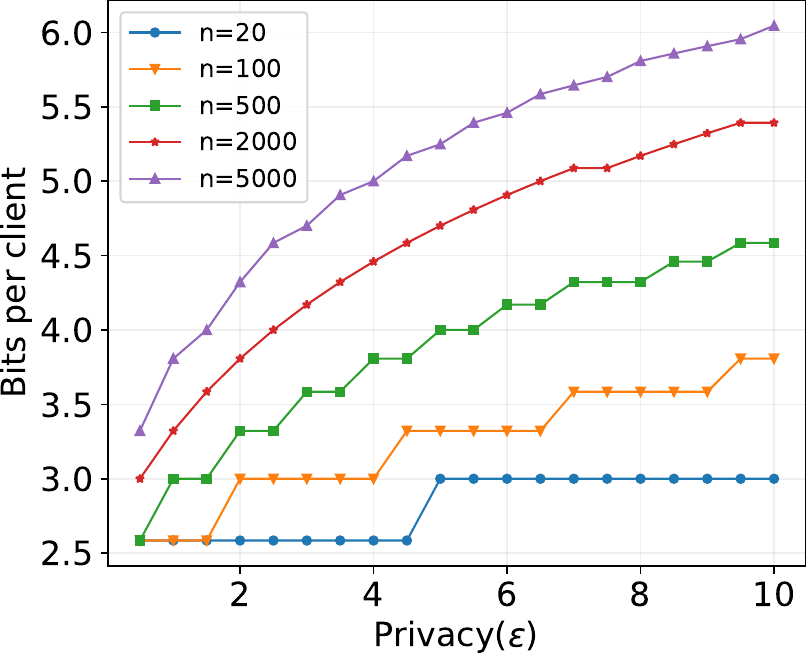}\label{fig:bits_shifted}}
  \subfigure[Shifted Layered]{
  \includegraphics[scale=0.38]{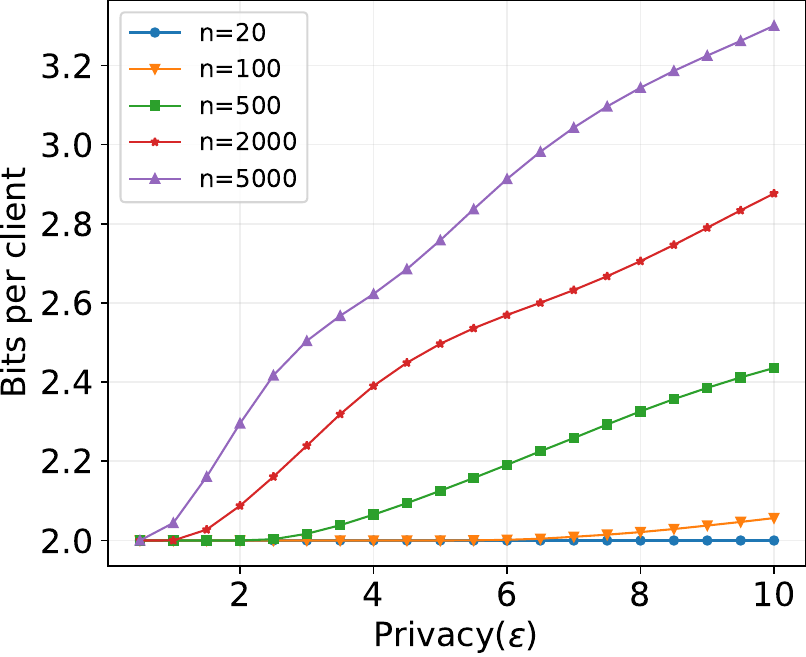}\label{fig:upper_bound_shifted}}
   \vspace{-0.33cm} \caption{Number of bits per clients for the Aggregation Gaussian mechanism (left), shifted layered quantizer with fixed (center) or variable (right) length and different client size $n \in \{20;100;500;2000;5000\}$.}
    \label{fig:shifted_compare}
\end{figure}

\subsection{Langevin Dynamics}
\label{app:langevin}

\subsubsection{Quantized Langevin Stochastic Dynamics}
We consider the same FL framework as \citet{vono2022qlsd} where the goal is to perform Bayesian inference on a parameter $\theta \in \rset^d$ with respect to a dataset $\mathcal{D}$. The posterior distribution is assumed to admit a product-form density with respect to the $d$-dimensional Lebesgue measure in the form
\begin{align*}
\vspace{-0.2cm}
    \pi(\theta|\mathcal{D}) = C_{\pi}^{-1} \prod_{i=1}^n e^{-U_i(\theta)}, \qquad C_\pi = \int_{\rset^d} \prod_{i=1}^n e^{-U_i(\theta)} \mathrm{d}\theta
    \vspace{-0.2cm}
\end{align*}
where $U_i$ are clients' potential functions. Using a sequence $(H_k)_{k \in \nset}$ of unbiased estimates of $\nabla U = \sum_{i=1}^n U_i$, the Langevin dynamics with stochastic gradient aims at sampling from a target distribution with density $\pi$. Starting from $\theta_0 \in \rset^d$, it is a Markov chain defined by the update rule: $\theta_{k+1} = \theta_k - \gamma H_{k+1}(\theta_k) + \sqrt{2 \gamma} Z_{k+1}$ where $\gamma>0$ is a discretization stepsize and $(Z_k)_{k \in \nset}$ is a sequence of i.i.d. standard Gaussian random variables.

In the framework of \citet{vono2022qlsd}, at iteration $k$, a subset  $\mathcal{A}_{k+1} \subset [n]$ of clients is selected. Each client $i \in \mathcal{A}_{k+1}$ computes an estimate of $H_{k+1, i}$ and sends $\mathscr{C}(H_{k+1,i} - \widetilde{H}_{k+1,i})$ to the server where $\mathscr{C}$ is a compression operator and $\widetilde{H}_{k+1,i}$ is a variance reduction term. The server aggregates the clients gradients estimators, potentially compensates for the variance reduction, and updates the process. Throughout the remaining of this subsection, we consider $\mathscr{C}$ to be a shifted layered quantizer with a fixed number of bits as described at the beginning of Appendix \ref{app:supp_numerical} and returning $x+\ncal(0,\sigma_b^2\norm{x}_\infty^2), \;\; \sigma_b^2\norm{x}_\infty^2  \leftarrow \mathscr{C}(x)$.

\subsubsection{Experiments on Gaussian}

We adapt this experiment from \citet{vono2022qlsd}. We use $n=20$ clients with data dimension $d=50$ and local potentials $U_i(\theta) = \sum_{j=1}^{N_i} \| \theta - y_{i,j} \|^2/2$ where $(y_{i,j})$ is a set of synthetic independent but not identically distributed observations across clients and $N_i=50$. The data is generated by sampling $y_{i,j} \sim \ncal(\mu_i, \mathrm{I}_d)$ with $\mu_i \sim \ncal(0,25\mathrm{I}_d)$. We set the discretization stepsize $\gamma=5\times 10^{-4} $ and  use full participation and full batch. We adapt the \texttt{QLSD}$^{\star}$ algorithm from \citet{vono2022qlsd} to leverage the error of compression in the Langevin dynamics. If the error of compression is smaller than the required error, then the server adds additional noise. Otherwise, the server adds no additional noise to the dynamics. The theoretical analysis in \citet{vono2022qlsd} still holds with this adaptation as it still satisfies the required assumptions on the compression operator (refer to their assumption \textbf{H3}).
\begin{figure}[h!]
    \centering
    \includegraphics[scale=0.45]{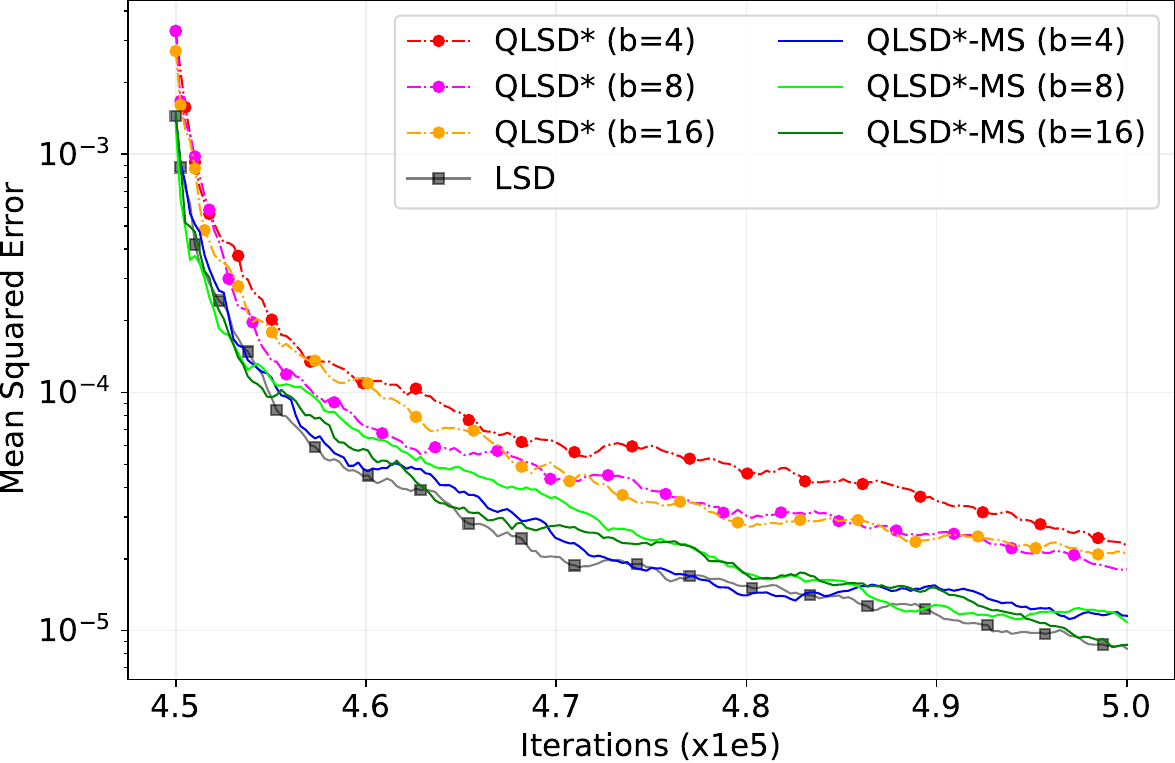}
    \vspace{-0.2cm}
    \caption{MSE of different methods, b refers to the number of bits, QLSD$^*$ to \texttt{QLSD}$^{\star}$ with unbiased quantization, QLSD$^*$-MS to the algorithm with shifted layered quantization, LSD to the algorithm with no compression}
    \label{fig:qlsd}
\end{figure}

Figure \ref{fig:qlsd} above shows the behavior of the mean squared error (MSE) between the true parameter and the sampling based estimation. We start sampling after $4.5\times 10^5$ iteration to be sure the chain converge to a stationary distribution. The results were computed using $30$ independent runs of each algorithm. Observe the great performance of shifted layered quantizers schemes. Indeed, all schemes using standard unbiased quantization performs worse than all schemes with shifted layered quantizers.

\begin{algorithm}[h]
   \caption{\texttt{QLSD}$^{\star}$ with Shifted Layered Quantizer}
   \label{algo:QLSD-star}
  \begin{algorithmic}
     \State {\bfseries Input:} minibatch sizes $\{n_i\}_{i \in [b]}$, number of iterations $K$, compression operators $\mathscr{C}$, step-size $\gamma \in (0,\bar{\gamma}]$ with $\bar{\gamma} > 0$ and initial point $\theta_0$.
     \For{$k=0$ {\bfseries to} $K-1$}
     \For{$i \in \mathcal{A}_{k+1}$ \Comment{On active clients}}
        \State Draw $\mathcal{S}_{k+1}^{(i)} \sim \mathrm{Uniform}\pr{\wp_{N_i,n_i}}$.
        \State {\small Set $H_{k+1}^{(i)}(\theta_k) = (N_i/n_i)\sum_{j \in \mathcal{S}_{k+1}^{(i)}} [\nabla U_{i,j}(\theta_k) - \nabla U_{i,j}(\theta^{\star})]$.}
        \State Compute $\textsl{g}_{i,k+1}, v_{i,k+1} \leftarrow
        \mathscr{C}\pr{H_{k+1}^{(i)}(\theta_k)}$.
        \State Send $\textsl{g}_{i,k+1}, v_{i,k+1}$  to the central server.
     \EndFor
     \State \Comment{On the central server}
     \State Compute $\textsl{g}_{k+1} = \frac{b}{|\mathcal{A}_{k+1}|}\sum_{i \in \mathcal{A}_{k+1}}\textsl{g}_{i,k+1}$.
     \State Draw $Z_{k+1} \sim \mathrm{N}(0_d,\mathrm{I}_d)$.
     \State Compute $\beta^2 = \max \open{0, 2\gamma- \frac{b^2\gamma^2}{|\mathcal{A}_{k+1}|^2}\sum_{i\in\mathcal{A}_{k+1} }v_{i,k+1}}$
     \State Compute $\theta_{k+1} = \theta_k - \gamma \textsl{g}_{k+1} + \beta Z_{k+1}$.
     \State Send $\theta_{k+1}$ to the $b$ clients.
     \EndFor
     \State {\bfseries Output:} samples $\{\theta_k\}_{k=0}^{K}$.
  \end{algorithmic}
\end{algorithm}

\section{Compression for Randomized Smoothing in Federated Learning}
\label{app:compression_smoothing}

In this section, we show how the proposed quantizers with exact error distribution can be applied to obtain optimal algorithms for non-smooth
distributed optimization problems \citep{scaman_optimal_2018}. In particular, we highlight the link between compression and randomized smoothing \citep{duchi_randomized_2012}. In the framework of federated learning with bi-directional compression \citep{philippenko2021preserved}, it allows us to derive fast convergence rates for non-smooth objectives. To the best of found knowledge, such rates for federated learning with client compression on non-smooth objectives are novel.

\textbf{Federated Learning with double compression.} Consider some optimization problems of the form $$\min_{\theta \in \rset^d} \left\{f(\theta)= \frac{1}{n} \sum_{i=1}^n f_i(\theta) \right\}$$ where $f_i$ are local losses and $f$ is a convex and potentially non-smooth objective function. This distributed setting includes FL problems with $\ell_1$ norms used as regularizers to ensure sparsity, \textit{e.g.} $f(\theta) = \frac{1}{n} \|A\theta - b\|_1 = \frac{1}{n} \sum_{i=1}^n |a_i^\top \theta - b_i|$ with $A \in \rset^{n \times d}, b \in \rset^n$ but also neural networks with ReLU activation function.

In the standard setting of Federated Averaging with bi-directional compression, the model parameter $\theta_k$ at time step $k \in \nset$ is updated as
\begin{align*}
    \theta_{k+1} = \theta_{k} - \gamma \mathscr{C}_{\downarrow}\left(\frac{1}{n}\sum_{i=1}^n  \mathscr{C}_{\uparrow}(g_i(\theta_k))\right)
\end{align*}
where $\gamma>0$ is the learning rate, $\mathscr{C}_{\uparrow}$ (resp. $\mathscr{C}_{\downarrow}$) is the up-link/client compressor (resp. the down-link/server) and $g_i(\theta_k)$ is a subgradient of the local loss $f_i$ such that $\expec[g_i(\theta_k)] \in \partial f_i(\theta_k)$.

\textbf{Distributed Randomized Smoothing.} Fast rates for non-smooth optimization problems can be attained using the smoothing approach of \citet{duchi_randomized_2012,scaman_optimal_2018}. For $\sigma>0$, denote by $f_{\sigma}$ the \textit{smoothed} version of $f$ defined by
\begin{align*}
    f_{\sigma}(\theta) = \expec_{\xi}[f(\theta + \sigma \xi)],
\end{align*}
where $\xi \sim \mathcal{N}(0,\mathrm{I}_d)$ is a standard Gaussian random variable. If $f$ is $L$-Lipschitz then $f_{\sigma}$ is $(L/\sigma)$ smooth and it holds (see Lemma E.3 in \citet{duchi_randomized_2012})
\begin{align*}
    \forall \theta \in \rset^d, \quad f(\theta) \leq f_{\sigma}(\theta) \leq  f(\theta) + \sigma L \sqrt{d}.
\end{align*}
Thus, accelerated optimization algorithms such as Distributed Randomized Smoothing (DRS) can be applied by using the \textit{smoothed} version of $f$. These algorithms rely on approximating the smoothed gradient $\nabla f_{\sigma} = \frac{1}{n} \sum_{i=1}^n \nabla f_{i,\sigma}$ with sampled subgradients of the form $g_i(\theta + \sigma \xi_j)$ with $j=1,\ldots,m$. Each local client $i=1,\ldots,n$ can sample standard Gaussian variables then compute $\hat g_i(\theta) = \frac{1}{m} \sum_{j=1}^m g_i(\theta + \sigma \xi_j)$ and send it to the server which aggregates the subgradients. Interestingly, the sampling steps may be replaced with compressors that produce exact error distribution.

\textbf{Quantizers act as randomized smoothing.} The idea of sampling a random perturbation $\xi$ to evaluate the subgradients as $g_i(\theta + \sigma \xi)$ can be replaced by first compressing the model parameter $\theta$ with a Gaussian error distribution as $\ecal(\theta) = \theta + \sigma \xi$ and then evaluating the subgradients at compressed point as $g_i(\ecal(\theta))$. Similarly to \citet{philippenko2021preserved}, let us consider the update rules where the subgradients are evaluted at \textit{perturbed} point as
\begin{align*}
    \theta_{k+1} = \theta_k - \frac{\gamma}{n} \sum_{i=1}^n \hat g_i(\hat \theta_k); \quad 
    \hat \theta_k = \ecal (\theta_k) = \theta_k + \xi_k,
\end{align*}
with a Gaussian error $\xi_k = \ecal (\theta_k) - \theta_k \sim \mathcal{N}(0,\sigma^2 \mathrm{I}_d)$. In view of approximating the sampling scheme of DRS, the local clients can perform $m$ local compressions so that $\hat g_i(\hat \theta_k) = \frac{1}{m} \sum_{j=1}^m g_i(\theta_k + \sigma \xi_j)$ which gives an unbiased estimate of the smoothed gradient as $\expec[\hat g_i(\hat \theta_k)|\theta_k] = \nabla f_{i,\sigma}$. Thus, one can exactly recover the Distributed Randomized Smoothing algorithm of \citet{scaman_optimal_2018} and the optimal convergence rates of Theorem 1 therein.

\end{document}